\documentclass{article}

\usepackage[main, final]{neurips_2026}

\usepackage[utf8]{inputenc} %
\usepackage[T1]{fontenc}    %
\usepackage{hyperref}       %
\usepackage{url}            %
\usepackage{booktabs}       %
\usepackage{amsfonts}       %
\usepackage{nicefrac}       %
\usepackage{microtype}      %
\usepackage{xcolor}         %

\usepackage{amsmath,amssymb,amsthm}
\usepackage{subcaption}
\usepackage{tikz}
\usetikzlibrary{arrows.meta,positioning,calc,decorations.pathreplacing}

\usepackage{algorithm}
\usepackage{algpseudocode}

\newtheorem{proposition}{Proposition}
\newtheorem{lemma}{Lemma}
\newtheorem{definition}{Definition}
\newtheorem{remark}{Remark}

\newtheorem{corollary}{Corollary}
\newtheorem{property}{Property}

\newcommand{\E}{\mathbb{E}}
\newcommand{\R}{\mathbb{R}}
\newcommand{\C}{\mathbb{C}}
\newcommand{\Var}{\mathrm{Var}}
\newcommand{\abs}[1]{\lvert #1 \rvert}
\newcommand{\norm}[1]{\lVert #1 \rVert}
\newcommand{\dt}{\,\mathrm{d}}
\newcommand{\CEOne}{\mathrm{CE}_1}
\newcommand{\CE}{\mathrm{CE}_2}
\newcommand{\CEhat}{\widehat{\mathrm{CE}}_2}
\newcommand{\F}{\mathcal{F}}
\newcommand{\Fl}{\mathcal{P}_l}
\newcommand{\sech}{\mathrm{sech}}
\newcommand{\etao}{\eta^\mathrm{orig}}
\DeclareMathOperator{\KL}{KL}
\DeclareMathOperator{\TV}{TV}
\DeclareMathOperator{\Bern}{Bernoulli}
\newcommand{\mt}{\widetilde m}
\newcommand{\St}{\widetilde S}
\newcommand{\sigt}{\widetilde{\sigma^2}}

\newcommand{\CEpert}{\mathrm{CE}_2^{\mathrm{pert}}}

\newcommand{\cF}{\mathcal F}

\title{The Minimax Rate of \\ Perturbed Second-Order Calibration}

\author{
  Kamil Ciosek\thanks{Correspondence to: \texttt{kciosek@gmail.com}} \quad
  Banafsheh Rafiee \quad
  Sina Ghiassian \quad
  Nicol\`{o} Felicioni \\
  Spotify
}

\begin{document}

\maketitle

\begin{abstract}
Second-order calibration error quantifies how closely a higher-order predictor's epistemic-uncertainty estimate matches the conditional variance of the label probability on its level sets. We characterize the minimax rate of estimating the second-order calibration error for binary classification in the regime where a small perturbation is applied to the classifier outputs. Our procedure is simple: add independent bandwidth-$h$ $\sech$ noise to the score coordinates, then regress $Y^{(1)}$ and $Y^{(1)}Y^{(2)}$ on the perturbed score using low-degree polynomials. Crucially, the $\sech$  perturbation makes the calibration functions \emph{analytic} in a suitable strip. The resulting estimator has error $O_h(\log^{3/2}n/\sqrt n)$, with explicit constants. In the same setting, a matching $\Omega(1/\sqrt{n})$ lower bound establishes minimax optimality up to logarithmic factors. As a corollary, we give a finite-sample guarantee for perturbed second-order Platt scaling, yielding a post-hoc procedure that recalibrates both the mean prediction and the epistemic-variance estimate of the perturbed higher-order predictor. Along the way, we give an explicit two-moment formulation of second-order calibration error and relate it quantitatively to the bucketed formulation of \citet{ahdritz2025provable}. Our experiments confirm the predicted rate and the quality of the recalibrated uncertainties.
\end{abstract}

\section{Introduction}
\label{sec:intro}
Modern predictors often output not just a probability but also an uncertainty estimate. A deep ensemble, a Bayesian neural network, or an evidential model produces, for an input $x$, a mean prediction $m(x)$ together with a variance-like disagreement statistic $\sigma^2(x)$ intended to capture epistemic uncertainty (for example the empirical variance in ensemble outputs). Whether $\sigma^2$ actually has this meaning is a separate question from how it was produced. \citet{ahdritz2025provable} gave a precise characterization: $\sigma^2$ captures epistemic uncertainty exactly when the predictor is \emph{higher-order calibrated}. A calibrated $\sigma^2$ separates inherent ambiguity (aleatoric uncertainty) from resolvable ignorance (epistemic uncertainty).  We propose to measure such second-order calibration (and thus the quality of epistemic uncertainty) using the second-order calibration error, in analogy to the well-known first-order calibration error \citep{blasiok2023, lee2023t}. 

This raises two basic statistical questions. How many 2-snapshots are needed to \emph{estimate} the second-order calibration error of a given
predictor, and at what rate? And how can one \emph{correct} a miscalibrated $(m,\sigma^2)$ so that its epistemic uncertainty acquires real-world semantics, with finite-sample guarantees? 
A fundamental problem is that some regularization is necessary: without
additional structure, the worst-case estimation error need not vanish as
more data are observed (see Appendix~\ref{appdx:assumptions-needed}).
\citet{ahdritz2025provable} address this by replacing exact level sets with a fixed partition of the score space.

We replace the partition by perturbing the classifier output using a $\sech$ kernel, an intervention that is controlled by just one hyperparameter (bandwidth) and that is arguably less invasive. Calibration functions are then approximated by regressing the labels on polynomial features of perturbed scores. Crucially, this specific perturbation makes the calibration functions of the perturbed classifier \emph{analytic} in a suitable complex strip. Leveraging standard tools, we then show that this polynomial regression estimates the calibration error at rate $\tilde{O}(1/\sqrt{n})$ with explicit constants; moreover, using the (standard) Le Cam technique, we provide a matching $\Omega(1/\sqrt{n})$ lower bound, proving minimax-optimality up to logarithmic factors. The same construction then yields a post-hoc recalibration procedure with a finite-sample second-order calibration guarantee for the perturbed predictor, which we call second-order Platt scaling.

\paragraph{Contributions.} We make the following contributions.
\begin{enumerate}
\item \textbf{Moment formulation}: we give an explicit binary two-moment
calibration error, defined by conditioning on the predictor's score rather
than choosing finite buckets.
\item \textbf{Analyticity} (lemma~\ref{thm:analytic}): for fixed bandwidth $h$, the sech-perturbed calibration functions are analytic in a complex strip.
\item \textbf{Fixed-$h$ estimation at $\tilde O(1/\sqrt n)$} (proposition~\ref{thm:main}): polynomial regression estimates $\CEpert$ at rate $O_h(\log^{3/2}n/\sqrt n)$ from 2-snapshots.
\item \textbf{Rate tightness for the same target}: proposition~\ref{prop:lb} gives an $\Omega(1/\sqrt n)$ lower bound for fixed-$h$ sech-perturbed calibration, matching the upper bound up to logarithmic factors.
\item \textbf{Second-order Platt scaling} (proposition~\ref{thm:recal}): remapping through the estimated calibration functions approximately recalibrates the perturbed predictor at the same $\tilde O_h(1/\sqrt n)$ rate.
\end{enumerate}

\section{Setup}
\label{sec:setup}

We study binary classification with $Y \in \{0,1\}$.  Nature generates i.i.d.\ data; $f^*(x) = P(Y{=}1 \mid X{=}x)$.  We make no assumptions on $f^*$ or the marginal of $X$ i.e., we are in the distribution-free setting. We have $n$ i.i.d.\ 2-snapshots $(X_i, Y_i^{(1)}, Y_i^{(2)})$ with $Y_i^{(1)}, Y_i^{(2)} \mid X_i \stackrel{\text{i.i.d.}}{\sim} \mathrm{Bernoulli}(f^*(X_i))$.

\begin{definition}[Higher-order predictor and calibration]
\label{def-ho-cal}
A higher-order predictor maps $x \mapsto S(x) = (m(x), \sigma^2(x))$ with $m \in [0,1]$, $\sigma^2 \in [0, m(1{-}m)]$.  Its calibration functions are
\begin{align}
  \eta_1(s) &= \E[Y \mid S(X) = s], \qquad \eta_2(s) = \E[f^*(X)^2 \mid S(X) = s], \label{eq:calfns}
\end{align}
and the second-order calibration error is 
\begin{gather} 
\CE = \E[\abs{\eta_1(S) - m}] + \E[\abs{\eta_2(S) - v}] 
\label{eq-ce-def}
\end{gather} 
with $v = m^2 + \sigma^2$.
\end{definition}

\paragraph{Intuitions.} We will now provide intuitions about the usefulness of definition~\ref{def-ho-cal}. The law of total variance can always be used to write
\begin{equation}
\label{eq:decomp}
  \underbrace{\eta_1(s)(1 - \eta_1(s))}_{\text{total}} = \underbrace{\E[f^*(X)(1{-}f^*(X)) \mid S{=}s]}_{\text{aleatoric}} + \underbrace{\Var(f^*(X) \mid S{=}s)}_{\text{epistemic}}.
\end{equation}
When $\CE = 0$ i.e., for a second-order calibrated classifier, we can further write
\begin{equation*}
    \Var(f^*(X) \mid S{=}s) = \underbrace{\E[(f^*(X))^2 \mid S{=}s]}_{\eta_2(s) = v} - {\underbrace{\E[f^*(X) \mid S{=}s]}_{\eta_1(s) = m}}^2 = v - m^2 = \sigma^2
\end{equation*}
This means that calibration allows us to interpret $\sigma^2$ as a measure of epistemic uncertainty. A high-level comparison between first and second-order calibration is given in Table~\ref{tab-overview}. Note that while the definition of second-order calibration given here is ostensibly different from the definition due to \citet{ahdritz2025provable}, the two versions are in fact very closely related (see Appendix~\ref{app:reconciliation} for a discussion). See Figure~\ref{fig:rubin-second-order} for further illustrations of calibration functions and how they relate to epistemic and aleatoric uncertainty.\footnote{The observation that ambiguity is connected to epistemic uncertainty has been made by \citet{osband2023epistemic} in the context of epistemic neural networks, a specific way of obtaining uncertainty estimates.}

\begin{figure}[t]
\centering
\begin{tikzpicture}[
    >=Latex,
    font=\small,
    box/.style={draw, rounded corners=2pt, minimum width=2.6cm, minimum height=2.6cm},
    score/.style={draw, rounded corners=2pt, align=left, inner sep=6pt, fill=gray!6},
    lab/.style={align=center},
    every node/.style={font=\small}
]

\node[box] (amb) at (0,0) {};
\node[box] (vase) at (6,0) {};
\node[box, opacity=0] (betw-node) at (7.5,0) {};
\node[box] (faces) at (9,0) {};

\begin{scope}[shift={(amb.center)}]
\fill[black] (-1.20,-1.20) rectangle (1.20,1.20);

\fill[white]
  (0.702,-1.093) --
  (-0.707,-1.088) --
  (-0.688,-1.005) --
  (-0.619,-0.944) --
  (-0.367,-0.935) --
  (-0.298,-0.893) --
  (-0.270,-0.837) --
  (-0.288,-0.735) --
  (-0.214,-0.679) --
  (-0.214,-0.623) --
  (-0.186,-0.586) --
  (-0.223,-0.530) --
  (-0.223,-0.502) --
  (-0.140,-0.456) --
  (-0.130,-0.372) --
  (-0.270,-0.177) --
  (-0.270,-0.112) --
  (-0.242,-0.056) --
  (-0.242,0.009) --
  (-0.298,0.158) --
  (-0.335,0.326) --
  (-0.409,0.456) --
  (-0.526,0.572) --
  (-0.609,0.628) --
  (-0.721,0.674) --
  (-0.870,0.693) --
  (-0.935,0.721) --
  (-1.005,0.809) --
  (-1.009,0.935) --
  (0.995,0.930) --
  (0.995,0.847) --
  (0.977,0.791) --
  (0.926,0.730) --
  (0.851,0.693) --
  (0.721,0.674) --
  (0.609,0.628) --
  (0.526,0.572) --
  (0.409,0.456) --
  (0.372,0.400) --
  (0.326,0.307) --
  (0.288,0.121) --
  (0.242,0.000) --
  (0.242,-0.056) --
  (0.270,-0.112) --
  (0.270,-0.177) --
  (0.130,-0.372) --
  (0.130,-0.437) --
  (0.163,-0.470) --
  (0.209,-0.488) --
  (0.223,-0.512) --
  (0.186,-0.586) --
  (0.214,-0.623) --
  (0.214,-0.679) --
  (0.288,-0.735) --
  (0.270,-0.828) --
  (0.288,-0.884) --
  (0.367,-0.935) --
  (0.600,-0.935) --
  (0.637,-0.953) --
  (0.698,-1.023) -- cycle;
\end{scope}
\node[lab, below=3mm of amb] {\textbf{ambiguous}\\Rubin figure};

\begin{scope}[shift={(vase.center)}]
\fill[white] (-1.20,-1.20) rectangle (1.20,1.20);
\fill[black]
  (0,0.92)
  .. controls (0.28,0.92) and (0.44,0.84) .. (0.44,0.74)
  .. controls (0.44,0.65) and (0.34,0.58) .. (0.22,0.54)
  -- (0.22,0.25)
  .. controls (0.22,0.10) and (0.54,-0.02) .. (0.66,-0.24)
  .. controls (0.76,-0.43) and (0.78,-0.68) .. (0.54,-0.88)
  .. controls (0.36,-1.03) and (0.16,-1.10) .. (0,-1.12)
  .. controls (-0.16,-1.10) and (-0.36,-1.03) .. (-0.54,-0.88)
  .. controls (-0.78,-0.68) and (-0.76,-0.43) .. (-0.66,-0.24)
  .. controls (-0.54,-0.02) and (-0.22,0.10) .. (-0.22,0.25)
  -- (-0.22,0.54)
  .. controls (-0.34,0.58) and (-0.44,0.65) .. (-0.44,0.74)
  .. controls (-0.44,0.84) and (-0.28,0.92) .. (0,0.92)
  -- cycle;

\fill[white] (0,0.74) ellipse [x radius=0.25, y radius=0.07];

\draw[line width=1.1pt]
  (0.30,0.50)
  .. controls (0.74,0.48) and (0.78,0.00) .. (0.46,-0.12)
  .. controls (0.34,-0.16) and (0.28,-0.10) .. (0.30,-0.02);

\draw[line width=1.1pt]
  (-0.30,0.50)
  .. controls (-0.74,0.48) and (-0.78,0.00) .. (-0.46,-0.12)
  .. controls (-0.34,-0.16) and (-0.28,-0.10) .. (-0.30,-0.02);

\end{scope}
\node[lab, below=3mm of vase] {\textbf{vase}\\only};

\begin{scope}[shift={(faces.center)}]
\fill[black] (-1.20,-1.20) rectangle (1.20,1.20);

\begin{scope}[shift={(-0.42,0.06)}, rotate=-12, line cap=round, line join=round]
  \filldraw[white,draw=black,line width=0.035]
    (0.00,0.98)
    .. controls (0.30,0.96) and (0.50,0.82) .. (0.58,0.54)
    .. controls (0.64,0.28) and (0.60,-0.02) .. (0.48,-0.38)
    .. controls (0.38,-0.68) and (0.22,-0.92) .. (0.00,-1.06)
    .. controls (-0.22,-0.92) and (-0.38,-0.68) .. (-0.48,-0.38)
    .. controls (-0.60,-0.02) and (-0.64,0.28) .. (-0.58,0.54)
    .. controls (-0.50,0.82) and (-0.30,0.96) .. (0.00,0.98)
    -- cycle;

  \fill[black]
    (-0.34,0.28)
    .. controls (-0.24,0.38) and (-0.10,0.38) .. (-0.02,0.28)
    .. controls (-0.10,0.18) and (-0.24,0.18) .. (-0.34,0.28)
    -- cycle;

  \fill[black]
    (0.34,0.28)
    .. controls (0.24,0.38) and (0.10,0.38) .. (0.02,0.28)
    .. controls (0.10,0.18) and (0.24,0.18) .. (0.34,0.28)
    -- cycle;

  \draw[black,line width=0.030]
    (-0.36,0.42) .. controls (-0.22,0.50) and (-0.08,0.50) .. (0.00,0.42);

  \draw[black,line width=0.030]
    (0.36,0.42) .. controls (0.22,0.50) and (0.08,0.50) .. (0.00,0.42);

  \draw[black,line width=0.030]
    (0.02,0.18) -- (-0.04,-0.08) -- (0.08,-0.11);

  \fill[black]
    (-0.22,-0.33)
    .. controls (-0.07,-0.16) and (0.07,-0.16) .. (0.22,-0.33)
    .. controls (0.10,-0.49) and (-0.10,-0.49) .. (-0.22,-0.33)
    -- cycle;
\end{scope}

\begin{scope}[shift={(0.42,-0.02)}, rotate=11, line cap=round, line join=round]
  \filldraw[white,draw=black,line width=0.035]
    (0.00,1.00)
    .. controls (0.28,0.98) and (0.48,0.86) .. (0.56,0.58)
    .. controls (0.63,0.32) and (0.59,0.02) .. (0.47,-0.34)
    .. controls (0.36,-0.68) and (0.20,-0.94) .. (0.00,-1.10)
    .. controls (-0.20,-0.94) and (-0.36,-0.68) .. (-0.47,-0.34)
    .. controls (-0.59,0.02) and (-0.63,0.32) .. (-0.56,0.58)
    .. controls (-0.48,0.86) and (-0.28,0.98) .. (0.00,1.00)
    -- cycle;

  \fill[black]
    (-0.34,0.30)
    .. controls (-0.24,0.38) and (-0.11,0.37) .. (-0.04,0.26)
    .. controls (-0.13,0.16) and (-0.25,0.17) .. (-0.34,0.30)
    -- cycle;

  \fill[black]
    (0.34,0.30)
    .. controls (0.24,0.38) and (0.11,0.37) .. (0.04,0.26)
    .. controls (0.13,0.16) and (0.25,0.17) .. (0.34,0.30)
    -- cycle;

  \draw[black,line width=0.030]
    (-0.36,0.45) .. controls (-0.22,0.40) and (-0.10,0.40) .. (0.00,0.44);

  \draw[black,line width=0.030]
    (0.36,0.45) .. controls (0.22,0.40) and (0.10,0.40) .. (0.00,0.44);

  \draw[black,line width=0.030]
    (0.00,0.20) -- (0.05,-0.08) -- (-0.06,-0.12);

  \fill[black]
    (-0.22,-0.24)
    .. controls (-0.08,-0.44) and (0.08,-0.44) .. (0.22,-0.24)
    .. controls (0.09,-0.16) and (-0.09,-0.16) .. (-0.22,-0.24)
    -- cycle;
\end{scope}
\end{scope}
\node[lab, below=3mm of faces] {\textbf{two faces}\\only};

\node[score, below=2.3cm of amb, text width=5.6cm] (leftscore) {
\textbf{Set A}: $S=(\frac12, 0)$\\[1mm]
\scriptsize
True label-probability: \(f^*(x)=1/2\).
\[
\eta_1=1/2,\qquad \eta_2=1/4
\]
hence
\[
\sigma^2=\eta_2-\eta_1^2=1/4-1/4=0.
\]
\begin{center}
(no epistemic uncertainty)    
\end{center}
};

\node[score, below=2.3cm of betw-node, text width=5.6cm] (rightscore) {
\textbf{Set B: $S=(\frac12, \frac14)$}\\[1mm]
\scriptsize
Half the time \(f^*(x)=1\), half the time \(f^*(x)=0\).
\[
\eta_1=1/2,\qquad \eta_2=\tfrac12(1^2+0^2)=1/2
\]
hence
\[
\sigma^2=\eta_2-\eta_1^2=1/2-1/4=1/4.
\]
\begin{center}
(high epistemic uncertainty)    
\end{center}
};

\draw[decorate,decoration={brace,amplitude=5pt,mirror},thick]
  ($(vase.south west)-(0,1.75)$) -- ($(faces.south east)-(0,1.75)$)
  node[midway,above=6pt,align=center] { };

\draw[decorate,decoration={brace,amplitude=5pt,mirror},thick]
  ($(amb.south west)-(0,1.75)$) -- ($(amb.south east)-(0,1.75)$)
  node[midway,above=6pt,align=center] { };

\end{tikzpicture}

\caption{Why second-order calibration is useful. Both sets have the same mean
prediction $m = 1/2$, so a first-order calibrated predictor treats them
identically. But Set~A's Rubin figures are genuinely 50:50 (aleatoric,
$\sigma^2 = 0$), while Set~B mixes clear vases with clear faces that the
classifier hasn't seen enough data to tell apart (epistemic, $\sigma^2 = 1/4$).
A calibrated $\sigma^2$ distinguishes inherent ambiguity from resolvable
ignorance.}
\label{fig:rubin-second-order}
\end{figure}
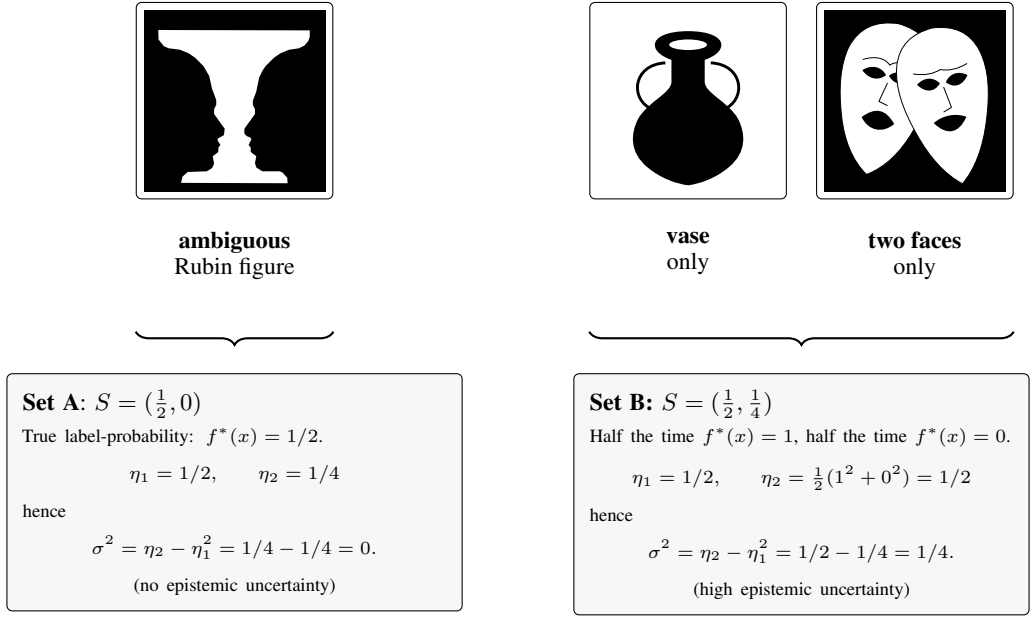

\paragraph{Goals.} Definition~\ref{def-ho-cal} applies to any higher-order predictor. Our statistical goal is to estimate this error for the fixed-$h$ sech-perturbed predictor (defined in section \ref{sec:analytic}), and to recalibrate that same predictor using the calibration functions in \eqref{eq:calfns}.

\begin{table}[t]
\begin{center}
\begin{tabular}{lll}
& \textbf{First-order} & \textbf{Second-order (this paper)} \\
\midrule
Output & probability $s$ & $(m, \sigma^2)$ \\
Data & labels $Y_i$ & 2-snapshots $(Y_i^{(1)}, Y_i^{(2)})$ \\
Measure & $\eta(s) = \E[Y \mid S{=}s]$ & $\eta_1(s), \eta_2(s)$ \\
Diagnose & $\mathrm{CE}_1 = \E[|\eta - s|]$ & $\CE$ \\
Fix & $s \mapsto \eta(s)$ & $(m,\sigma^2) \mapsto (\eta_1, \eta_2 - \eta_1^2)$ \\
Guarantee & calibrated probabilities & calibrated epistemic uncertainty \\
Rate & $\tilde O(1/\sqrt{n})$ & $\tilde O(1/\sqrt{n})$ \\
\end{tabular}
\end{center}
\caption{Comparison of first and second-order calibration (rates hold for the $\sech$-perturbed predictor).}
\label{tab-overview}
\end{table}

\section{The Sech Perturbation, Analyticity and Estimation}
\label{sec:analytic}
\paragraph{Key insight: analyticity.} We crucially leverage the bandwidth-$h$ $\sech$ perturbation kernel of \citet{ciosek2026measuring}, previously only used to enforce smoothness (we define the kernel formally in Definition~\ref{def:perturb} in the following section). Crucially, this specific perturbation makes the resulting calibration functions analytic in a complex strip. The classical Bernstein ellipse theorem then gives exponentially accurate polynomial approximation, yielding the $O_h(\log^{3/2}n/\sqrt n)$ bound with polynomial degree $l=O_h(\log n)$. We note that, while Definition~\ref{def-ho-cal} is universal, our bounds assume that the classifier outputs have the $\sech$ perturbation applied (with a fixed bandwidth, see the following section).

\subsection{The perturbation scheme}
\label{sec:perturb}

Similarly to \citet{ciosek2026measuring}, we perturb the predictor's scores by adding noise from a $\sech$ kernel.  This is a design choice: the perturbed predictor is the one we deploy and recalibrate. \citet{ciosek2026measuring} have shown that, for realistic values of the bandwidth $h$, this perturbation has a very small impact on classifier performance (as measured by AUROC). We describe a practical process for choosing the bandwidth $h$ in Appendix~\ref{appdx:choose-h}.

\begin{definition}[Sech perturbation]
\label{def:perturb}
For a score $s \in [a,b]$ and bandwidth $h > 0$, the perturbation kernel is
\begin{equation}
\label{eq:kernel}
  k_h(t \mid s) = \frac{\sech\!\left(\frac{t-s}{h}\right)}{Z(s,h)} \cdot \mathbf{1}_{[a,b]}(t), \qquad Z(s,h) = \int_a^b \sech\!\left(\frac{u-s}{h}\right)\dt u.
\end{equation}
The perturbed score $S \sim k_h(\cdot \mid s)$ satisfies $S \in [a,b]$.
\end{definition}

Write
\[
S_{\mathrm{orig}}(x)
=
\bigl(m_{\mathrm{orig}}(x),\sigma_{\mathrm{orig}}^2(x)\bigr)
\]
for the original predictor. After perturbation, we write
$S(x)=(m(x),\sigma^2(x))$. Throughout the remainder of the paper,
{\bf unadorned notation refers to this perturbed, deployed score}.
We apply a product perturbation on
$\mathcal R=[0,1]\times[0,\tfrac14]$; that is, we draw the two coordinates
independently:
\begin{enumerate}
\itemindent1cm 
\item[\textbf{Step 1.}] Draw
$m\sim k_h(\cdot\mid m_{\mathrm{orig}})$ on $[0,1]$.
\item[\textbf{Step 2.}] Draw
$\sigma^2\sim k_h(\cdot\mid\sigma_{\mathrm{orig}}^2)$ on
$[0,\tfrac14]$, independently.
\end{enumerate}

The rectangle $\mathcal R$ contains the feasible region
\[
    \F=\{(m_{\mathrm{orig}},\sigma_{\mathrm{orig}}^2):
    \sigma_{\mathrm{orig}}^2
    \le m_{\mathrm{orig}}(1-m_{\mathrm{orig}})\},
\]
so every original score lies in $\mathcal R$. The perturbed score
$(m,\sigma^2)$ may fall outside $\F$, but always remains in $\mathcal R$.
For perturbed scores outside $\F$, we use the same conditional moments
and the same error formula \eqref{eq-ce-def}, now on all of $\mathcal R$.
All estimation and recalibration guarantees below concern this perturbed
score.

\subsection{The analyticity lemma}

The following lemma, proved in Appendix~\ref{sec-lemma-analytic}, is the
kernel-specific step in our analysis. It shows that the denominator of the
normalized sech smoothing is zero-free in a complex strip, and hence that
the perturbed calibration functions are analytic. The subsequent polynomial approximation and regression arguments use standard tools. This zero-free argument does not automatically extend to Gaussian perturbations, because analyticity of the numerator and denominator alone does not exclude complex zeros of the denominator.

\begin{lemma}[Analyticity of perturbed calibration functions]
\label{thm:analytic}
Let $S_{\mathrm{orig}}$ be any higher-order predictor with calibration
functions $\etao_1,\etao_2$. Let $\eta_1,\eta_2$ be the calibration
functions of its $\sech$-perturbed predictor with bandwidth $h>0$.

\begin{enumerate}
\item[\textbf{(1D)}] If $S_{\mathrm{orig}}(x)=m_{\mathrm{orig}}(x) \in [0,1]$ is a scalar score, then $\eta_1$ is analytic in the strip $\{z \in \C : |\mathrm{Im}(z)| < h\pi/2\} \cap \{z \in \mathbb{C} : 0 \le \mathrm{Re}(z) \le 1\}$.
\item[\textbf{(2D)}] For the full score $S_{\mathrm{orig}}(x)
=(m_{\mathrm{orig}}(x),\sigma_{\mathrm{orig}}^2(x))$ with product perturbation, fix any $0<\rho<h\pi/4$. Then, for each $j=1,2$, $\eta_j$ extends jointly holomorphically to $\Omega_\rho:=\{z\in\C^2: |\operatorname{Im}z_1|<\rho,\ |\operatorname{Im}z_2|<\rho\}$. Moreover, on the closed product sub-strip $|\operatorname{Im}z_r|\le\rho$, we have $|\eta_j(z)|\le M_\rho$, where $M_\rho:=\sec(2\rho/h)$.
\end{enumerate}
\end{lemma}

\subsection{Polynomial Estimation}
\label{sec:estimation}
We state a classical polynomial-approximation result, proved in
Appendix~\ref{sec-proof-bernstein} in our notation.
\begin{lemma}[Bernstein ellipse theorem on \mbox{$\left[0,1\right]$}]
\label{prop:bernstein}
Let $f : [0,1] \to \R$ be the restriction of a function analytic in the strip $|\mathrm{Im}(z)| < \rho$ around $[0,1]$, with $\sup_{|\mathrm{Im}(z)| < \rho} |f(z)| \le M$.  Then for every integer $l \ge 0$, there exists a polynomial $q_l$ of degree $l$ with
\[
  \sup_{t \in [0,1]} |f(t) - q_l(t)| \le 2M\,\frac{\theta^{-l}}{\theta - 1},
\]
where $\theta = 2\rho + \sqrt{4\rho^2 + 1} > 1$ is the Bernstein ellipse parameter.
\end{lemma}
This yields the following corollary (proven in Appendix~\ref{sec-proof-cor})
\begin{corollary}[Polynomial approximation of $\eta$]
\label{cor:approx}
Fix any $0<\rho<h\pi/4$, and define
\[
        \theta_\rho:=2\rho+\sqrt{4\rho^2+1},
        \qquad
        M_\rho:=\sec(2\rho/h).
\]
For every integer $\ell\ge0$ and each $j=1,2$, there exists a polynomial
$q_{j,\ell}(m,\sigma^2)$ of degree at most $\ell$ in each coordinate such
that, on $\mathcal R=[0,1]\times[0,1/4]$,
\[
   0\le q_{j,\ell}(m,\sigma^2)\le1,
   \qquad
   \sup_{(m,\sigma^2)\in\mathcal R}
   |\eta_j(m,\sigma^2)-q_{j,\ell}(m,\sigma^2)|
      \le B_\rho\theta_\rho^{-\ell},
\]
where $B_\rho:=\frac{16M_\rho\theta_\rho}{(\theta_\rho-1)^2}.$
\end{corollary}

Corollary~\ref{cor:approx} immediately gives us a way to leverage analyticity to measure the quality of the polynomial approximation in our specific use-case. Indeed, for fixed bandwidth $h$, we choose $\rho=h\pi/8$, so that $\theta_h:=\frac{h\pi}{4}+\sqrt{1+\frac{h^2\pi^2}{16}}>1$, $M_h=\sqrt2$, and $B_h:=\frac{16\sqrt2\,\theta_h}{(\theta_h-1)^2}$ are constants. Choosing $l=\left\lceil2\log n/\log\theta_h\right\rceil$ gives $\varepsilon_l\le B_h\theta_h^{-l}\le B_hn^{-2}$, negligible compared to the estimation error.

\section{Our Algorithm and Estimator}
\label{sec:alg}
Algorithm~\ref{alg:ce-estimator} summarizes our core procedure. The constrained
polynomial regression class is defined below. 
\begin{algorithm}[t]
\caption{Estimating fixed-$h$ perturbed second-order calibration error}
\label{alg:ce-estimator}
\begin{algorithmic}[1]
\Require Original predictor $S_{\mathrm{orig}}$,
2-snapshot data $\{(X_i,Y_i^{(1)},Y_i^{(2)})\}_{i=1}^n$,
bandwidth $h$, and polynomial degree $l$

\For{$i=1,\ldots,n$}
    \State $(m_{\mathrm{orig},i},\sigma_{\mathrm{orig},i}^2)
    \gets S_{\mathrm{orig}}(X_i)$
    \State Draw
    $m_i\sim k_h(\cdot\mid m_{\mathrm{orig},i})$ on $[0,1]$
    \State Draw
    $\sigma_i^2\sim k_h(\cdot\mid\sigma_{\mathrm{orig},i}^2)$
    on $[0,\tfrac14]$, independently
    \State $S_i\gets(m_i,\sigma_i^2)$,
    $P_i\gets Y_i^{(1)}Y_i^{(2)}$, and
    $v_i\gets m_i^2+\sigma_i^2$
\EndFor

\State Fit $\hat\eta_1$ by least squares of $Y_i^{(1)}$ on $S_i$
over the class $\Fl$
\State Fit $\hat\eta_2$ by least squares of $P_i$ on $S_i$
over the class $\Fl$
\State $\displaystyle
\CEhat\gets\frac1n\sum_{i=1}^n
\left(
\abs{\hat\eta_1(S_i)-m_i}
+
\abs{\hat\eta_2(S_i)-v_i}
\right)$
\State \Return $\CEhat,\hat\eta_1,\hat\eta_2$
\end{algorithmic}
\end{algorithm}

For any bounded function $u$ on $\mathcal R$, write $ \|u\|_\infty :=\sup_{(m,\sigma^2)\in\mathcal R}|u(m,\sigma^2)|$.
Let $\Fl$ denote the class
\[
   \Fl
   :=
   \left\{
      \mathrm{clip}_{[0,1]}(p):
      p(m,\sigma^2)=\beta^\top\phi(m,\sigma^2),
      \ \|p\|_\infty\le1
   \right\},
\]
where $\phi$ is the tensor product polynomial basis of degree $l$ in
each variable, so $\phi$ has $d=(l+1)^2$ components. The unit ball of
this finite-dimensional polynomial space is compact, and clipping is
continuous in the sup norm. Therefore $\Fl$ is compact and the
squared-loss ERM minimizers below exist. Every function in $\Fl$ takes
values in $[0,1]$.

Given 2-snapshots $(X_i, Y_i^{(1)}, Y_i^{(2)})_{i=1}^n$ with perturbed scores $S_i = (m_i, \sigma^2_i)$ and products $P_i = Y_i^{(1)} Y_i^{(2)}$:
\begin{align}
  \hat\eta_1 &= \arg\min_{f \in \Fl}\; \frac{1}{n}\sum_{i=1}^n \bigl(Y_i^{(1)} - f(S_i)\bigr)^2, \label{eq:erm1}\\
  \hat\eta_2 &= \arg\min_{f \in \Fl}\; \frac{1}{n}\sum_{i=1}^n \bigl(P_i - f(S_i)\bigr)^2. \label{eq:erm2}
\end{align}
Both regressions use the same constrained ERM class $\Fl$, with responses
$Y_i^{(1)}$ and $P_i$, respectively. Proposition~\ref{thm:main} analyzes
this constrained ERM (the experiments instead use ridge regression followed by output clipping as a practical surrogate).

The CE estimator is:
\begin{equation}
\label{eq:cehat}
  \CEhat
=
\frac{1}{n}\sum_{i=1}^n
\abs{\hat\eta_1(S_i)-m_i}
+
\frac{1}{n}\sum_{i=1}^n
\abs{\hat\eta_2(S_i)-(m_i^2+\sigma_i^2)}.
\end{equation}

\section{Main Theoretical Result}
\label{sec:main}
We now characterize the estimation rate for the calibration error of the fixed-$h$ perturbed predictor.
\begin{proposition}[Estimation of $\CEpert$ at near-parametric rate]
\label{thm:main}
Fix the perturbation bandwidth $h > 0$ and let $\theta_h:=\frac{h\pi}{4}+\sqrt{1+\frac{h^2\pi^2}{16}}$. Set the polynomial degree to $l=\left\lceil 2\log n/\log\theta_h\right\rceil$, so that $d=(l+1)^2=O_h(\log^2 n)$. Then the estimator~\eqref{eq:cehat} satisfies
\begin{equation}
\label{eq:rate}
  \E\bigl[\abs{\CEhat - \CE^{\mathrm{pert}}}\bigr] = O_h\!\left(\frac{\log^{3/2} n}{\sqrt{n}}\right),
\end{equation}
where $\CE^{\mathrm{pert}}$ is the second-order calibration error of the perturbed predictor. The constant depends on $h$, but not on the score distribution or on the original calibration functions.
\end{proposition}

\paragraph{Origin of the logarithmic gap.}
For fixed $h$, analyticity gives approximation error that decreases exponentially with the polynomial degree $l$. Thus $l=O_h(\log n)$ makes the approximation error negligible. In two dimensions, the tensor polynomial class then has
\[
    d=(l+1)^2=O_h(\log^2 n)
\]
coefficients. The covering-number bounds used for regression and for evaluating the empirical calibration error contribute a further $\log n$
factor. The resulting estimation terms are therefore
\[
    O\!\left(\sqrt{\frac{d\log n}{n}}\right)
    =
    O_h\!\left(\frac{\log^{3/2}n}{\sqrt n}\right).
\]
The lower bound is only $\Omega(1/\sqrt n)$. We do not know whether the $\log^{3/2}n$ factor is information-theoretically necessary; it may be an
artifact of the proof technique.

Moreover, we can use essentially the same technique to obtain a novel result about first-order calibration; this is reflected in the remark below.
\begin{remark} [First-Order Calibration]
While Proposition~\ref{thm:main} does not immediately concern first-order
calibration,\footnote{In first-order calibration, conditioning is only with
respect to the mean, ignoring the epistemic variance.} the same technique
estimates
$\CEOne=\E[\abs{\E[Y\mid m]-m}]$
at rate $O_h(\log^{3/2}n/\sqrt n)$.
\end{remark}

\paragraph{Comparison with bucketing.}
The bucketing rate depends on the specific setting we are considering. For a predictor that is constant on a
fixed partition with $B$ cells, the natural estimator of the aggregate,
probability-weighted second-order calibration error has expected error
$O(\sqrt{B/n})$. Thus, when $B$ is fixed, the partition-based target has the usual $n^{-1/2}$ rate, with a constant depending on $B$. A different bucketing rate of $n^{-1/4}$ arises in a different setting: using an increasingly fine regular partition to approximate an unpartitioned Lipschitz target in two dimensions. 
Our theorem concerns a
third target, $\CEpert$ for fixed $h$. The sech perturbation makes this target
analytic, giving
$O_h(\log^{3/2}n/\sqrt n)$. We therefore claim introducing the new perturbed setting rather than faster estimation in the fixed-partition setting.

\section{Lower Bound and Minimax Optimality for a Fixed Perturbation}
To complement the upper bound, we give a lower bound for the same fixed-$h$ sech-perturbed setting. Propositions~\ref{thm:main} and~\ref{prop:lb} therefore characterize its minimax rate up to logarithmic factors. The following proposition is proved in Appendix~\ref{app:lowerboundproof}.

\begin{proposition}[Minimax lower bound for fixed-$h$ perturbed $\CE$]\label{prop:lb}

There exists an absolute constant $c > 0$ such that, for every sech bandwidth
$h>0$, every $n \geq 1$, and every estimator $\CEhat$, two distributions
$P_0, P_1$ in the setup of the main paper exist with
\[
  \max_{b\in\{0,1\}}\;
    \E_{P_b}\!\bigl[\,\bigl|\CEhat - \CEpert(P_b)\bigr|\,\bigr]
  \;\geq\; \frac{c}{\sqrt n}.
\]
The constant $c$ does not depend on $h$.
\end{proposition}

\section{Second-Order Platt Scaling}
\label{sec:platt}
We now describe the second-order analogue of Platt scaling.

\subsection{The recalibration map}

Given the true calibration functions $\eta_1,\eta_2$, define the
\emph{recalibrated predictor}:
\begin{equation}
\label{eq:recal}
  T(x) = \bigl(\,\underbrace{\eta_1(S(x))}_{m'(x)},\;\; \underbrace{\eta_2(S(x)) - \eta_1(S(x))^2}_{(\sigma^2)'(x)}\,\bigr),
\end{equation}
where $S(x)$ is the perturbed score.

\begin{proposition}[Exact second-order Platt scaling]
\label{thm:recal}
If $\eta_1, \eta_2$ are the true calibration functions, then $T$ is perfectly second-order calibrated with respect to its own level sets:
\begin{align}
  \E[Y \mid T(X) = (m', (\sigma^2)')] &= m', \label{eq:recal_m}\\
  \E[f^*(X)^2 \mid T(X) = (m', (\sigma^2)')] &= (m')^2 + (\sigma^2)'. \label{eq:recal_v}
\end{align}
\end{proposition}

\begin{proof}
\textbf{First moment.}  By the tower property:
\[
  \E[Y \mid T(X) = (m', (\sigma^2)')]
  = \E[\,\underbrace{\E[Y \mid S(X)]}_{= \eta_1(S(X)) = m'(X)} \mid T(X) = (m', (\sigma^2)')].
\]
On the event $\{T(X) = (m', (\sigma^2)')\}$, the first component $m'(X) = m'$ is constant.  So the expression equals $m'$.

\textbf{Second moment.}  Similarly:
\begin{align*}
  \E[f^*(X)^2 \mid T(X) = (m', (\sigma^2)')]
  &= \E[\eta_2(S(X)) \mid T(X) = (m', (\sigma^2)')].
\end{align*}
Now $\eta_2(S(X)) = \eta_1(S(X))^2 + (\eta_2(S(X)) - \eta_1(S(X))^2) = m'(X)^2 + (\sigma^2)'(X)$.  On the conditioning event, this equals $(m')^2 + (\sigma^2)'$.
\end{proof}

\subsection{Approximate guarantee from finite data}
For the estimated calibration functions from Section~\ref{sec:alg},
we first enforce consistency of the two predicted moments:
\begin{equation}
\label{eq:That}
\begin{aligned}
a(s)&=\hat\eta_1(s), \qquad b(s)=\hat\eta_2(s),\\
b'(s)&=\min\{a(s),\max\{a(s)^2,b(s)\}\},\\
\hat T(x)&=\bigl(a(S(x)),\,b'(S(x))-a(S(x))^2\bigr).
\end{aligned}
\end{equation}
Since $0\le a(s)\le1$, clipping the second moment to
$[a(s)^2,a(s)]$ ensures that the reported variance lies in
$[0,a(s)(1-a(s))]$. Thus $\hat T$ always satisfies the feasibility
requirement in Definition~\ref{def-ho-cal}.

\begin{corollary}
\label{cor:approx_recal}
Under the assumptions of Proposition~\ref{thm:main}, the recalibrated
predictor $\hat T$ satisfies
\[
\E[\CE(\hat T)]
=
O_h\!\left(\frac{\log^{3/2}n}{\sqrt n}\right).
\]
\end{corollary}
See Appendix~\ref{appdx:cor-recal} for proof.

\section{Experiments}
\label{sec:experiments}

Our four experiments illustrate the method's empirical behavior.
Experiment~1 shows favorable finite-sample scaling of the practical
estimator; Experiment~2 confirms that the resulting
recalibrator $\hat T$ improves calibration in the spirit of
Corollary~\ref{cor:approx_recal}; Experiment~3 isolates \emph{why} that calibration is worth
doing: its \emph{magnitude}, not just its ranking, controls downstream decisions; and
Experiment~4 reproduces the same effect on real labels. Experiments~1-3 use synthetic data with
known ground truth, Experiment~4 uses the public Weather Sentiment AMT dataset.

\paragraph{Experiment 1: the predicted rate.}
The DGP is a 10-component Gaussian mixture in $\mathbb{R}^4$ with $f^*(x){=}\sigma(w^\top x)$;
the predictor is a 5-member MLP ensemble whose mean and unbiased variance are sech-perturbed
coordinate-wise. Ground-truth $\mathrm{CE}_2^{\mathrm{pert}}$ is computed by quasi-Monte Carlo
(replicate spread ${<}\,2{\cdot}10^{-5}$). We compare 2D bucketing, Nadaraya-Watson with an
isotropic Gaussian kernel, and our polynomial estimator (held-out and CV-tuned variants);
baseline bin counts and bandwidths follow textbook MISE rules with constants tuned on a held-out
split. Figure~\ref{fig:exp1} shows that the polynomial estimator dominates at every $n$, with
empirical log--log slope ${\approx}\,{-}0.70$ at both $h\!\in\!\{1/16,1/64\}$ - at least as steep as the
predicted $-1/2$ rate and steeper than the $-0.33$ to $-0.55$ slopes of bucketing and NW (consistent
with their distribution-free Lipschitz rate). The two polynomial variants are
indistinguishable, so the advantage is not an artefact of held-out tuning. Setup details in
App.~\ref{app:exp1}.

\begin{figure}[t]
\centering
\includegraphics[width=0.47\linewidth]{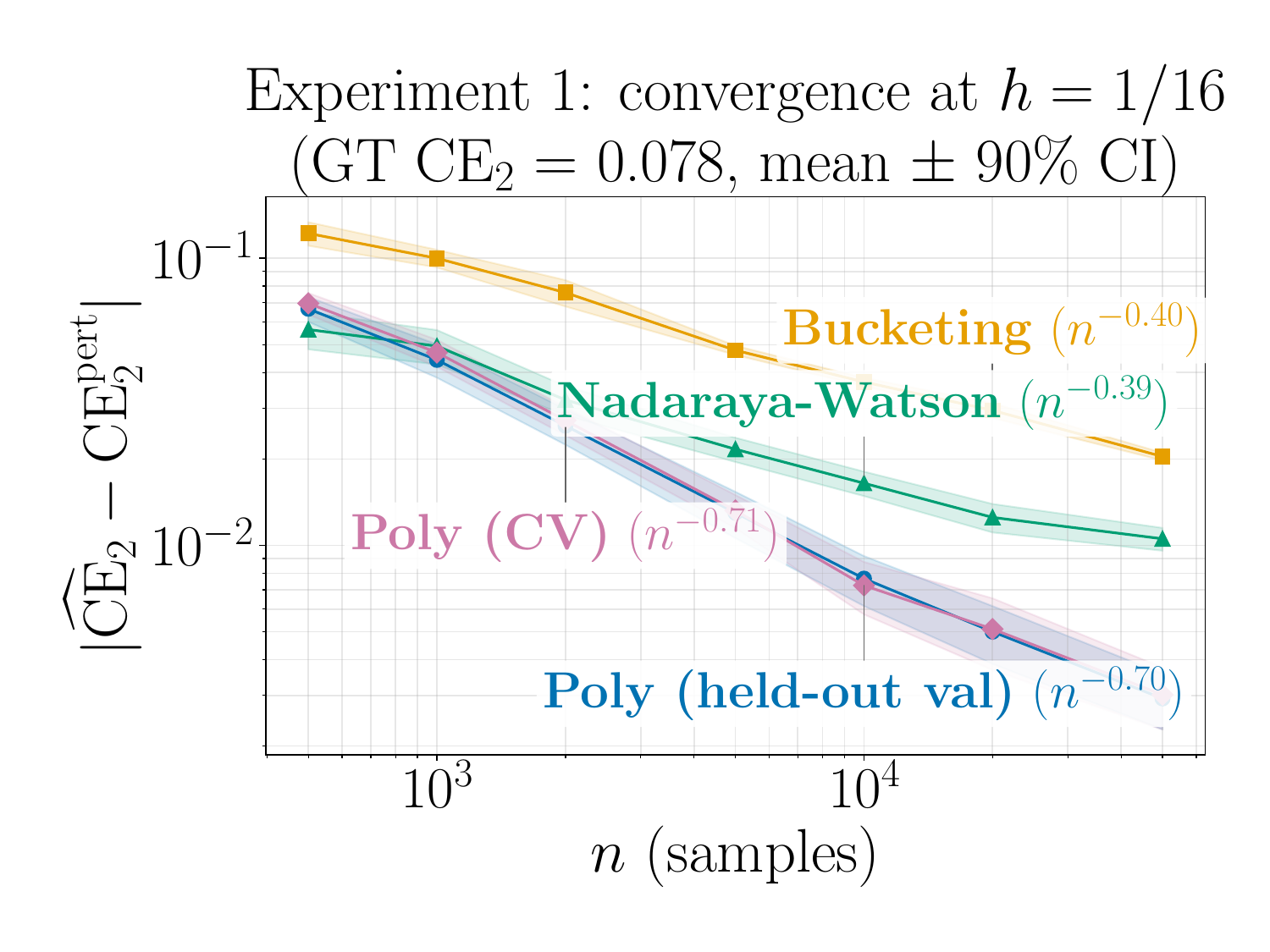}\hfill
\includegraphics[width=0.47\linewidth]{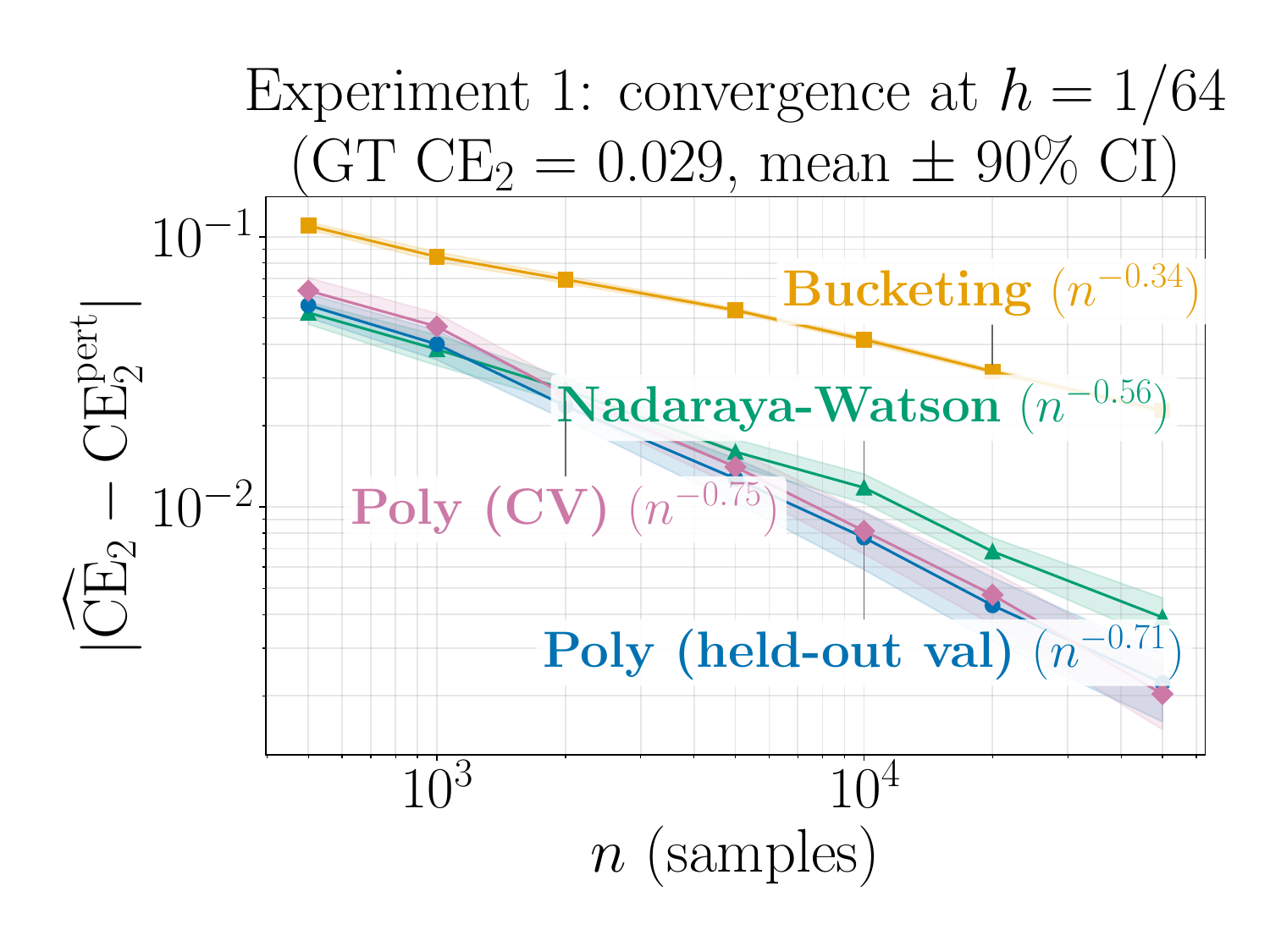}
\caption{\textbf{Experiment~1 (rate).} $|\widehat{\mathrm{CE}}_2-\mathrm{CE}_2^{\mathrm{pert}}|$
vs.\ $n$ at $h{=}1/16$ (left) and $h{=}1/64$ (right); mean $\pm$ Student-$t$ 90\% CI across 20
seeds, log--log slopes in legend.}
\label{fig:exp1}
\end{figure}

\paragraph{Experiment 2: the recalibrated $\sigma^2$ tracks the truth.}
We replace the predictor with a 20-member ensemble deliberately undertrained so that its raw
$\sigma^2$ carries only a faint epistemic signal; the recalibrator is a tensor-Chebyshev ridge
regression of degree~12 fit on $2{\cdot}10^4$ 2-snapshots. Ground-truth conditional variance at
perturbed scores is obtained by sech-smoothing the empirical $(p,p^2)$ joint of the DGP. The raw
$\sigma^2$ is essentially uncorrelated with truth (Pearson $0.10$), but after applying $\hat T$
the scatter concentrates on the diagonal (Pearson $0.74$); see Figure~\ref{fig:exp2}. Across 20
calibration seeds, $\mathrm{CE}_2$ drops from $0.37$ to $0.04\pm0.01$,
demonstrating improved calibration with the practical recalibrator.
Setup details in App.~\ref{app:exp2}.

\begin{figure}[t]
\centering
\begin{subfigure}{0.44\linewidth}\includegraphics[width=\linewidth]{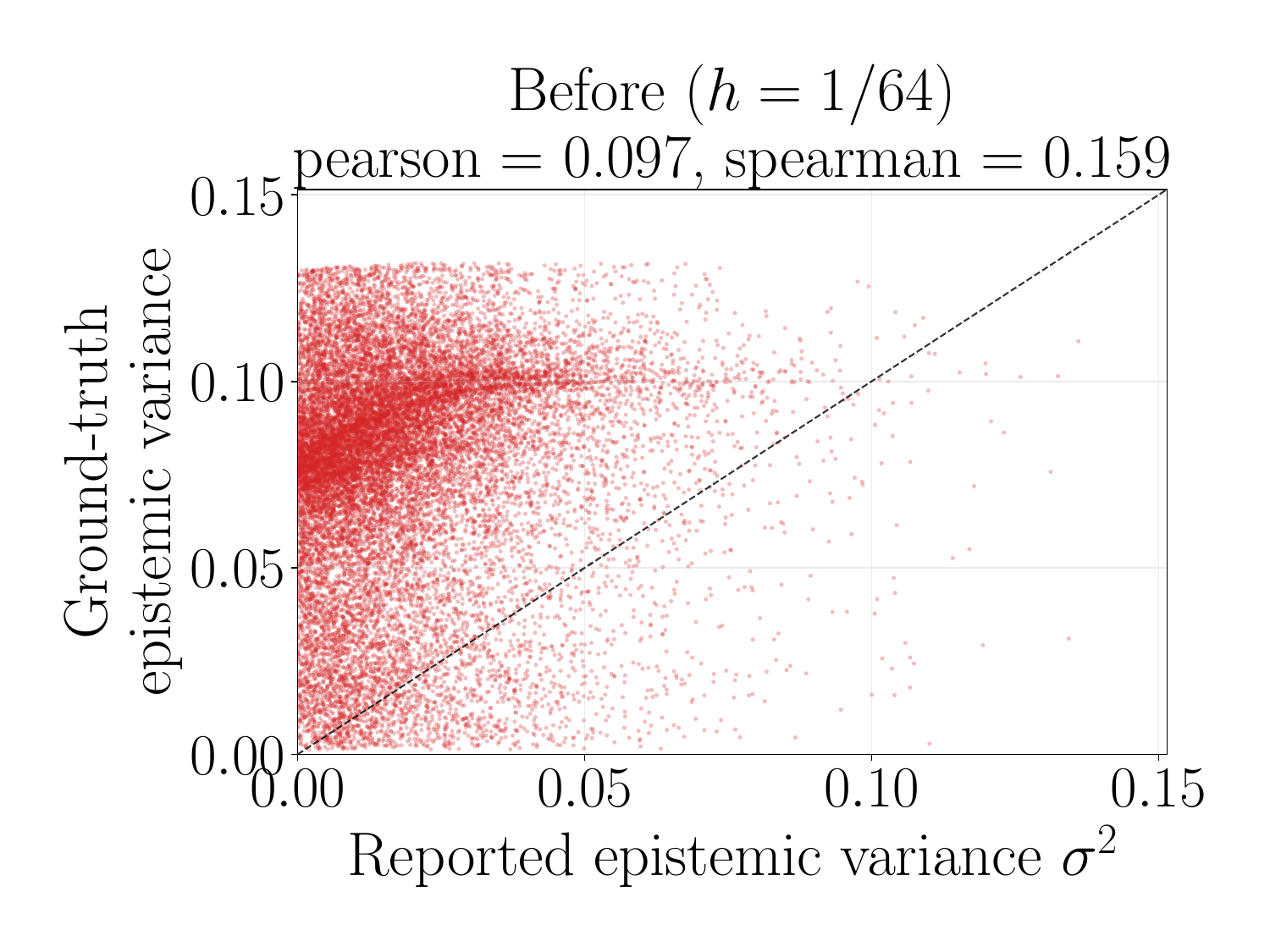}\end{subfigure}\hfill
\begin{subfigure}{0.44\linewidth}\includegraphics[width=\linewidth]{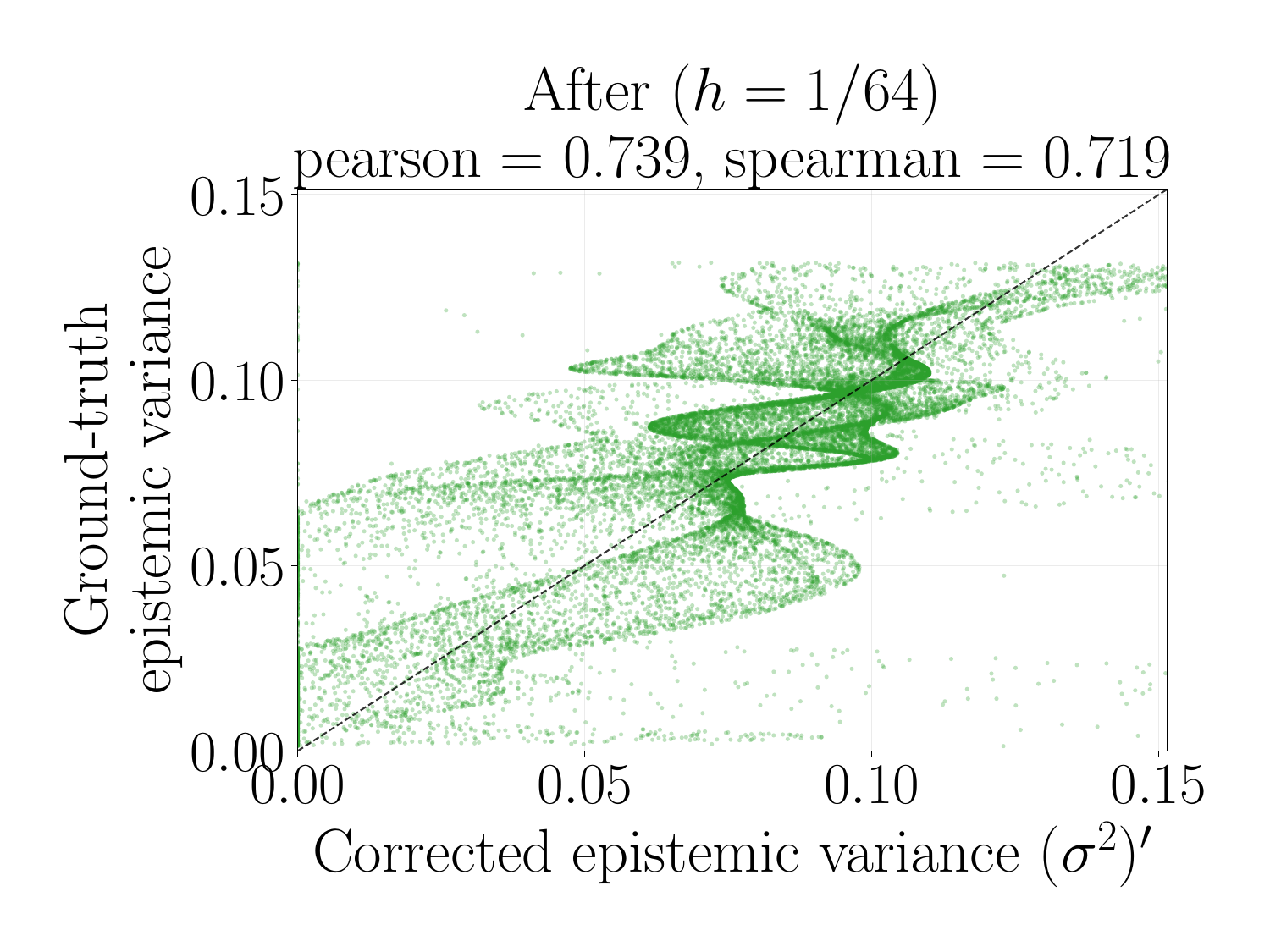}\end{subfigure}
\caption{\textbf{Experiment~2 (recalibration).} Reported $\sigma^2$ (left) and corrected
$(\sigma^2)'$ (right) versus the ground-truth conditional variance on a held-out split, at
$h{=}1/64$. Dashed line is $y{=}x$.}
\label{fig:exp2}
\end{figure}

\paragraph{Experiment 3: magnitudes, not just rankings.}
A clinic refers borderline patients ($m\!\in\![0.35,0.65]$) at cost $c{=}0.06$ for a noisy
specialist panel; the per-patient gain $g(\theta){=}2\theta^2-2\theta+0.5-c$ has conditional
mean $2\eta_2(S)-2\eta_1(S)+0.5-c$, so any post-hoc procedure that calibrates only $\eta_1$ is
misspecified for the Bayes-optimal score. The borderline cohort splits into an aleatoric
subpopulation ($\theta{\approx}0.5$, small $\sigma^2$) and a hidden-subtype one
($\theta\!\in\!\{0.12,0.88\}$, large $\sigma^2$), both centred at $m{\approx}0.5$.
Figure~\ref{fig:exp3} reports realised gain versus referral threshold $\tau$. 2-D second-order
Platt tracks the oracle across the full $\tau\!\geq\!0$ range. The raw $(m,\sigma^2)$ plug-in
matches at $\tau{=}0$ but \emph{collapses} near $\tau{\approx}0.06$: it ranks patients correctly
but its score \emph{magnitudes} are biased low on the epistemic group, so any cost-aware
threshold drops the entire group. First-order methods cannot separate the subpopulations and refer
essentially no one. The lesson is exactly what Proposition~\ref{thm:recal} promises and ranking-only
metrics would hide: a calibrated $\sigma^2$ matters numerically. Setup details in App.~\ref{app:exp3}.

\begin{figure}[t]
\centering
\begin{minipage}[c]{0.46\linewidth}\centering\includegraphics[width=\linewidth]{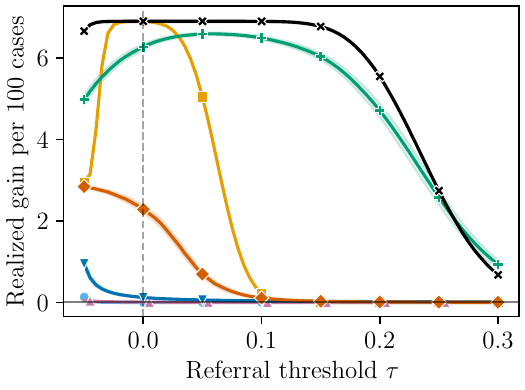}\end{minipage}\hfill
\begin{minipage}[c]{0.5\linewidth}\centering\includegraphics[width=\linewidth]{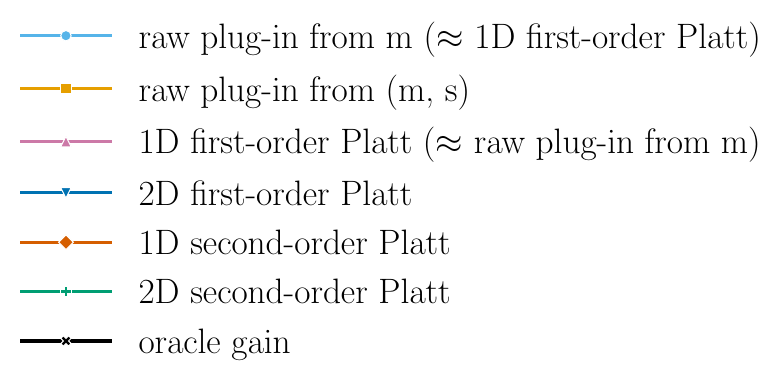}\end{minipage}
\caption{\textbf{Experiment~3 (decision utility).} Realised gain per 100 borderline patients
vs.\ referral threshold $\tau$ (mean $\pm 1.96\,$SEM, 200 repeats). 2-D second-order Platt tracks
the oracle; the raw $(m,\sigma^2)$ plug-in collapses near $\tau{\approx}0.06$.}
\label{fig:exp3}
\end{figure}

\paragraph{Experiment 4: the same effect on real labels.}
We move to the public Weather Sentiment AMT dataset (300 tasks, 20 worker votes each,
one-vs-rest binarisation), with an audit task whose downstream value provably depends on
$\eta_2$ and not just $\eta_1$: an auditor must rank items by the probability that two fresh
workers disagree, $g(\theta){=}2\theta(1-\theta)$, whose calibrated form $2\eta_1(S)-2\eta_2(S)$
requires $\eta_2$. Items are mixed in equal proportion from an aleatoric cohort
($\theta\!\approx\!1/2$) and a hidden-subtype cohort ($\theta$ near $0$ or $1$); the score
$S=(m,s)\in[0,1]^2$ is constructed in an XOR pattern so that the 1-D marginals of $m$ and of $s$
are \emph{identical} between cohorts, leaving any 1-D calibrator information-theoretically blind.
At a $5\%$ audit budget, 2-D second-order Platt finds $46.4$ disagreements per 100 audited items,
versus $30.9$ for 2-D first-order Platt, ${\sim}26$ for all five 1-D and raw baselines, and
$49.8$ for the oracle; the paired AUC lift of 2-D second-order over each non-oracle baseline has
a $95\%$ CI strictly above zero (per-baseline win fractions $\geq 0.97$). On real labels, in a
task whose value provably depends on $\eta_2$, second-order Platt is the only practical method
that approaches the oracle. Full audit-yield curves and per-method statistics in
App.~\ref{app:exp4} (Fig.~\ref{fig:exp4}).

\section{Related Work}
\label{sec:related}

\paragraph{Higher-order calibration.}
\citet{ahdritz2025provable} introduced higher-order calibration on predictor
level sets and developed partition-based methods using $k$-snapshots.
Their Theorem~2.6 provides a finite-sample guarantee for estimating the
snapshot distribution within a partition cell. We study the average binary
two-moment error of a fixed-$h$ sech-perturbed predictor under global
i.i.d.\ sampling. Appendix~\ref{app:reconciliation} relates our moment
criterion to their Wasserstein criterion on the same fixed partition.

\paragraph{First-order calibration measurement.}
ECE is conventionally estimated by binning
\citep{naeini2015obtaining,guo2017calibration}, but binned estimators are
biased and inconsistent
\citep{vaicenavicius2019evaluating,kumar2019verified}. \citet{kumar2019verified}
give a debiased estimator with finite-sample bounds, \citet{roelofs2022mitigating}
empirically compare variants, and \citet{widmann2019calibration} obtain
consistent multi-class estimators via matrix-valued kernels.
\citet{blasiok2023} unify this literature through the distance to calibration,
identifying which measures are polynomially related to it. Closest in spirit
to our work is the SmoothECE of \citet{blasiok2023smooth}, a first-order
calibration measure built from Nadaraya--Watson regression with a
reflected-Gaussian kernel. We differ in two essential respects: we study the
two-dimensional second-order error, and we use the $\sech$ kernel of
\citet{ciosek2026measuring}, which makes the fixed-$h$ perturbed calibration
functions analytic rather than merely smooth. For this perturbed target,
Bernstein-ellipse approximation gives exponentially accurate polynomial
approximation and an $O_h(\log^{3/2}n/\sqrt n)$ estimation rate. By contrast,
generic two-dimensional local regression under only Lipschitz regularity has
the $n^{-1/4}$ bias-variance rate described in Section~\ref{sec:main}. This comparison concerns
estimators of the same perturbed target, not the fixed-partition target of
\citet{ahdritz2025provable}. \citet{ciosek2026measuring} establish an
$n^{-1/3}$ rate for the first-order error and do not exploit analyticity.
\citet{lee2023t} study the complementary problem of \emph{testing} first-order
calibration and prove minimax optimality of their T-Cal procedure under a
smoothness assumption.

\paragraph{Distribution-free impossibility and regularization.}
\citet{gupta2020distribution} show that distribution-free calibration
guarantees require the score level sets to form at most a countable
partition, ruling out generic continuous-output schemes (under the assumptions of their framework).
For the related evaluation problem, \citet[Proposition~9]{lee2023t} show that
first-order calibration testing is impossible under continuity alone,
without stronger regularity. The reduction in Appendix~\ref{appdx:assumptions-needed}
transfers this obstruction to unrestricted unperturbed second-order
calibration. Bucketing, structural assumptions on the calibration functions,
and deployment-time perturbation are alternative ways to make the problem
tractable; our results use the last of these.

\paragraph{Post-hoc calibration.}
Platt scaling \citep{platt1999probabilistic}, isotonic regression
\citep{zadrozny2002transforming}, beta calibration \citep{kull2017beta}, and
temperature scaling \citep{guo2017calibration} are the standard first-order
recalibration methods: each fits a low-dimensional map from scores to
corrected probabilities. Our second-order Platt scaling is, to our knowledge, the first post-hoc procedure that recalibrates both a mean prediction and an epistemic-variance estimate with provable finite-sample bounds on the calibration error of the recalibrated classifier without requiring score bucketing.

\paragraph{Multicalibration.}
A complementary line of work strengthens first-order calibration along the
subgroup axis rather than the moment axis. \citet{hebert2018multicalibration}
require the mean prediction to be calibrated simultaneously on every
subpopulation in a rich, computationally-identifiable class. Our work is
orthogonal: we calibrate the second moment on the predictor's own level sets
rather than the first moment across many subgroups. Combining the two, i.e., developing multicalibrated higher-order predictors is a natural direction left to
future work.

\paragraph{Additional Related Work} We provide additional pointers to related work in appendix \ref{sec-more-related-work}.

\section{Conclusion}

For every fixed bandwidth, the sech perturbation makes the calibration functions of the perturbed predictor analytic. Exploiting this structure with polynomial regression yields $O_h(\log^{3/2}n/\sqrt n)$ estimation of $\CEpert$, with a matching $\Omega(1/\sqrt n)$ lower bound up to logarithmic factors. The result concerns the perturbed predictor that is evaluated, recalibrated, and deployed; it is distinct from both unperturbed calibration and calibration relative to a fixed partition. Moreover, remapping the perturbed predictor through the estimated calibration functions produces an approximately second-order calibrated predictor whose epistemic uncertainty has finite-sample guarantees. We briefly explore remaining open questions in Appendix~\ref{app:openquestions}.

\bibliographystyle{plainnat}
\bibliography{refs}

\appendix

\section{Further Related Work}
\label{sec-more-related-work}
\paragraph{Second-order uncertainty measures.}
A complementary line of work asks \emph{what} second-order quantity ought to
be measured, rather than how to calibrate it. \citet{wimmer2023quantifying}
examine entropy-based decompositions of aleatoric and epistemic uncertainty,
\citet{sale2023second} propose Wasserstein-based alternatives, and
\citet{hofman2024quantifyingaleatoricepistemicuncertainty} develop
scoring-rule decompositions. These works concern what to measure; ours
concerns how to calibrate it. The general aleatoric/epistemic distinction we
exploit is reviewed in \citet{hullermeier2021aleatoric} and
\citet{Hllermeier2019AleatoricAE}.

\paragraph{Bayesian and evidential uncertainty decomposition.}
Several lines of work produce mean-and-uncertainty outputs without
calibration guarantees: Bayesian active learning by disagreement
\citep{houlsby2011bayesian}, variance decompositions for Bayesian neural
networks \citep{depeweg2018decomposition,kendall2017uncertainties}, and deep
ensembles \citep{lakshminarayanan2017simple}. A complementary family bypasses
ensembling by directly parameterising a higher-order output distribution:
prior networks \citep{malinin2018predictive} and evidential deep learning
\citep{sensoy2018evidential} both place a Dirichlet on the simplex and read
epistemic uncertainty from its concentration. Second-order Platt scaling is
agnostic to the source of $(m, \sigma^2)$: it can recalibrate the output of
any of these methods post hoc.

\paragraph{Distribution-free uncertainty quantification.}
Our analysis is distribution-free in the sense that no assumption is placed
on $f^*$ or on the marginal of $X$; the only structural input is the sech
perturbation, which is a deployment-side design choice rather than an
assumption on the data. A different distribution-free paradigm is conformal
prediction \citep{vovk2005algorithmic,angelopoulos2023conformal}, which
produces per-sample prediction sets with a marginal coverage guarantee on the
label. Conformal set size reflects total predictive uncertainty and does not,
on its own, separate aleatoric from epistemic components; recent work that
attempts such a split \citep{sale2025aleatoric} does so by importing an
external decomposition and using conformal machinery to calibrate it.
Second-order Platt scaling is complementary: it targets the moments of the
conditional label distribution on the predictor's level sets, yielding a
calibrated epistemic-uncertainty scalar rather than a coverage-valid set.

\section{Details of Experiments}
\subsection{Experiment 1: Additional Details}
\label{app:exp1}

\paragraph{DGP and predictor.} $X\in\mathbb{R}^4$ is drawn from a 10-component
isotropic Gaussian mixture (centres $\sim\mathcal{N}(0,1.5^2 I)$, unit-variance
components); $f^*(x)=\sigma(w^\top x)$ with $w\sim\mathcal{N}(0, I/2)$. The
predictor is a 5-member MLP ensemble ($4\!\to\!64\!\to\!64\!\to\!1$, ReLU,
BCE, Adam at $10^{-3}$, 30 epochs, batch 256) trained once on a fresh sample
of $5{,}000$; the score is its mean and unbiased variance, clipped to the
feasible region. For each coordinate, with original score $s$ and
truncation interval $[a,b]$, draw $U\sim\mathrm{Unif}[0,1]$ and set
\[
\alpha=\arctan\!\left(e^{(a-s)/h}\right),\qquad
\beta=\arctan\!\left(e^{(b-s)/h}\right),\qquad
T=s+h\log\tan\!\left(\alpha+U(\beta-\alpha)\right).
\]
Then $T$ has density $k_h(\cdot\mid s)$ on $[a,b]$.
The two coordinates use independent draws of $U$.

\paragraph{Ground truth.} For each $h$, $\mathrm{CE}_2^{\mathrm{pert}}$ is
computed by depositing a scrambled-Sobol QMC sample of size $2^{18}$ on a
$1025\times 257$ tensor grid, applying the truncated-sech kernel matrices
exactly, and integrating the pointwise loss against the perturbed-score
density (trapezoidal rule). Two independent Sobol replicates agree to
$\le 2\!\cdot\!10^{-5}$, much smaller than any estimator error.

\paragraph{Baselines.} All estimators target $\eta_1$ by regressing $Y^{(1)}$
on the perturbed score and $\eta_2$ by regressing $Y^{(1)}Y^{(2)}$.
\emph{Bucketing} uses $K\times K$ axis-aligned cells with
$K=\mathrm{clip}(\lceil c\,(n/h^2)^{1/4}\rceil, 2, 200)$, the standard
2D-histogram MSE balance for an $O(1/h)$-Lipschitz target. \emph{Nadaraya--Watson}
uses an isotropic Gaussian kernel (truncated at $5b$) of bandwidth
$b = c\,h^{1/2}n^{-1/4}$ on the $\sigma^2$-rescaled rectangle (multiplying
$\sigma^2$ by 4 so both axes have unit support). The constant $c$ is selected
from $\{0.25, 0.5, 1, 2, 4, 8\}$ on a held-out split by squared-error loss;
bucketing dyadically extends the grid if the optimum lies on the boundary.

\paragraph{Practical polynomial estimator.}
The experiments use the ridge-plus-clipping surrogate described after
\eqref{eq:erm2}, not the constrained ERM analyzed in
Proposition~\ref{thm:main}. We use tensor Chebyshev features of degree $l$, fit by ridge with $\lambda$ scaled by $\mathrm{tr}(G)/d$, switching
to dual form when $(l+1)^2>n$. The candidate degree schedule starts at
$l_{\max}=\min(l_{\text{paper}}, L_{\text{cap}})$ with
$l_{\mathrm{paper}}(n,h)
=\lceil2\log n/\log\theta_h\rceil$, where
\[
\theta_h
=
\frac{h\pi}{4}
+
\sqrt{1+\frac{h^2\pi^2}{16}},
\]
using the theoretical degree from Proposition~\ref{thm:main} as an upper
cap, and per-$h$ caps
$L_{\text{cap}}\in\{44, 88\}$ for $h\in\{1/16,1/64\}$, then halves down to a
floor of $4$. The ridge multiplier is selected from $\{10^{-6},\dots,1\}$.
\emph{val}: $(l,\alpha)$ chosen on a held-out hyper split. \emph{ridge, CV}:
5-fold CV on the validation split, with closed-form leave-one-out
predictions reported on the same split. Per-coordinate predictions clipped
to $[0,1]$.

\paragraph{Sampling and aggregation.} For each of 20 seeds we draw three
i.i.d.\ splits (train, hyper, valid) of size $50{,}000$; perturbations are
redrawn per $h$, raw scores and labels are shared across $h$. For each
$n\in\{500, 10^3, 2\!\cdot\!10^3, 5\!\cdot\!10^3, 10^4, 2\!\cdot\!10^4,
5\!\cdot\!10^4\}$ the estimators use only the first $n$ samples of each
split (so per-seed $n$-curves are coupled). Bands are Student-$t$ 90\%
intervals across seeds; reported slopes are least-squares fits in
$\log_{10}(n)$ vs.\ $\log_{10}(\text{mean error})$.

\subsection{Experiment 2: implementation details}
\label{app:exp2}

\paragraph{Ensemble.} 20 MLPs (\texttt{4-64-64-1}, ReLU), each trained for
$60$ epochs of BCE with Adam (lr~$10^{-3}$, batch~$256$) on the same $24$
$(x,y)$ pairs from the DGP, varying only the initialisation seed. The small
per-member training set keeps the disagreement variance informative. Outputs
are clipped to the feasible set $\sigma^2\le m(1{-}m)$.

\paragraph{Ground-truth surfaces.} For each $h$, we draw $2^{18}$ scrambled
Sobol points from the DGP and bilinearly deposit $(1,p,p^2)$ on a
$1025\times 257$ tensor grid over $[0,1]\times[0,1/4]$. Convolving each axis
with the 1D sech kernel at bandwidth $h$ yields gridded $\eta_1,\eta_2$ and a
perturbed score density; the ground-truth conditional variance and
$\mathrm{CE}_2$ at $h$ are read off these surfaces, with bilinear interpolation
at arbitrary perturbed scores.

\paragraph{Recalibrator.} $(M,\Sigma)$ is affinely mapped to $[-1,1]^2$ and
lifted to a tensor-Chebyshev basis of degree~$12$ ($169$ features). The two
regressions for $\hat\eta_1,\hat\eta_2$ share this basis and are solved by
Cholesky on the $169\times 169$ Gram matrix, with only a numerical ridge floor
of $10^{-12}\cdot\mathrm{tr}(G)/169$ (no cross-validation, no degree search).
Predictions are projected onto the feasible moment region
$\hat\eta_1\in[0,1]$, $\hat\eta_2\in[\hat\eta_1^2,\hat\eta_1]$ before forming
$(\sigma^2)'$.

\paragraph{Variance across seeds.} The $\mathrm{CE}_2$ intervals reported in
the main text are mean $\pm$ standard deviation over $20$ independent draws of
the $n{=}2{\cdot}10^4$ calibration set.

\subsection{Experiment 3 details}
\label{app:exp3}

\paragraph{DGP.} Four-group mixture with weights $(0.25, 0.25, 0.35, 0.15)$.
Groups $0$/$1$ are clear negatives/positives concentrated at
$\theta \approx 0.08$ and $\theta \approx 0.92$ (small $\sigma^2$); they
contribute negligible mass after the borderline filter and exist only so the
unfiltered population has full support over $m$. The two groups that populate
the borderline cohort:
\begin{itemize}\setlength{\itemsep}{0pt}
  \item \emph{Group 2 (aleatoric, weight $0.35$):}
    $m \sim \mathcal{N}(0.5,\, 0.035^2)$,\;
    $\sigma^2 \sim \mathcal{N}(0.010,\, 0.002^2)$,\;
    $\theta \sim \mathcal{N}(0.5,\, 0.040^2)$.
  \item \emph{Group 3 (hidden subtype, weight $0.15$):}
    $m \sim \mathcal{N}(0.5,\, 0.035^2)$,\;
    $\sigma^2 \sim \mathcal{N}(0.060,\, 0.010^2)$;\;
    subtype $b \sim \mathrm{Unif}\{0,1\}$, then
    $\theta \sim \mathcal{N}(0.12,\, 0.030^2)$ if $b=0$, else
    $\mathcal{N}(0.88,\, 0.030^2)$.
\end{itemize}
All draws are independently clipped so that $m,\theta \in [0,1]$ and
$\sigma^2 \in [0, 1/4]$. Within the borderline cohort, $\eta_1 \approx 0.5$
for both groups while $\eta_2 \approx 0.25$ versus $\eta_2 \approx 0.4$, so
the two cannot be separated from $m$ alone.

\paragraph{Calibrators.} Tensor-product Chebyshev features
($T_0,\dots,T_d$ on $2m-1 \in [-1,1]$ and on $8\sigma^2-1 \in [-1,1]$) with
$d_{\mathrm{1D}}=6$ and per-axis $d_{\mathrm{2D}}=4$ ($25$ features), fit by
ridge ERM with ridge $10^{-4}$. Calibration labels are 2-snapshots
$Y^{(1)}, Y^{(2)} \stackrel{\mathrm{iid}}{\sim} \mathrm{Bernoulli}(\theta)$;
first-order calibrators regress $Y^{(1)}$ and set
$\hat\eta_2 := \hat\eta_1^2$, second-order calibrators jointly regress
$(Y^{(1)},\,Y^{(1)}Y^{(2)})$. Predicted moments are projected onto
$\{\hat\eta_1 \in [0,1],\ \hat\eta_1^2 \le \hat\eta_2 \le \hat\eta_1\}$
(Jensen plus the Bernoulli upper bound $\theta^2 \le \theta$). Each method's
patient score is $2\hat\eta_2 - 2\hat\eta_1 + 0.5 - c$.

\paragraph{Methods.} Seven curves in Figure~\ref{fig:exp3}: (i) raw plug-in
from $m$ ($\hat\eta_1\!:=\!m$, $\hat\eta_2\!:=\!m^2$); (ii) raw plug-in from
$(m,\sigma^2)$ ($\hat\eta_1\!:=\!m$, $\hat\eta_2\!:=\!m^2+\sigma^2$, then
project); (iii)/(iv) 1D/2D first-order Platt; (v)/(vi) 1D/2D second-order
Platt; (vii) oracle, scoring with the true $\theta$.

\paragraph{Replication.} $200$ independent repeats with seeds
$0,\dots,199$; pre-filter sizes $N_{\mathrm{cal}}=3000$,
$N_{\mathrm{eval}}=8000$ (about half pass the borderline filter under the
group weights); threshold grid of $71$ points in $[-0.05, 0.30]$. Bands in
Fig.~\ref{fig:exp3} are $\pm 1.96\,$SEM across repeats. We did not tune the
polynomial degrees, ridge, or filter range; the values above were our first
choice. Source code is included with the submission.

\subsection{Experiment 4: details}\label{app:exp4}

\paragraph{Data.} Weather Sentiment AMT: per-(worker, task) categorical
responses with majority-vote gold. We binarise as one-vs-rest on the gold
class; the figure uses class~1. Tasks with fewer than $5$ votes are dropped,
leaving $300$ tasks with $20$ votes each.

\paragraph{Cohort, provisional label, and $\theta$.} Let $\hat p_i$ be the
empirical positive rate of the binary votes for item $i$. The aleatoric cohort
is $\{i:|\hat p_i-\tfrac12|\le 0.22\}$, the hidden-subtype cohort is
$\{i:\hat p_i\le 0.12 \text{ or } \hat p_i\ge 0.88\}$, and the two are balanced
to equal size. Each item is assigned a provisional binary label
$\mathrm{prov}_i$: aleatoric items use $\mathrm{Bernoulli}(\tfrac12)$; hidden
items use the empirical majority for half (chosen uniformly) and its complement
for the other half. The agreement probability of a fresh worker is then
$\theta_i=\hat p_i$ if $\mathrm{prov}_i=1$, else $1-\hat p_i$. Aleatoric items
therefore have $\theta_i\!\approx\!\tfrac12$, hidden items $\theta_i$ near $0$
or $1$; both cohorts have $\mathbb{E}\theta_i\!\approx\!\tfrac12$ but very
different $\mathbb{E}\theta_i^2$.

\paragraph{Score construction (XOR mode).} Let
$\mathrm{lo}=\tfrac12-s_{\mathrm{gap}}/2$,
$\mathrm{hi}=\tfrac12+s_{\mathrm{gap}}/2$ with $s_{\mathrm{gap}}=0.35$. Draw
$b_i\sim\mathrm{Bernoulli}(\tfrac12)$ and set $m_i=\mathrm{hi}$ if $b_i=1$ else
$\mathrm{lo}$; set $s_i=\mathrm{hi}$ if
$\mathbb{1}[\text{kind}_i=\text{aleatoric}]=b_i$, else $\mathrm{lo}$. Add
independent Gaussian noise of std $\sigma_m=0.06$, $\sigma_s=0.08$ and clip to
$[0,1]$. By construction the marginal of $m$ and the marginal of $s$ are the
same Gaussian-noisy two-mode mixture in both cohorts, while the joint pattern
is a cohort-aligned XOR.

\paragraph{Calibrators.} For each calibration item, $\widehat\eta_1$ is fit to
$a_i/n_i$ and $\widehat\eta_2$ to $a_i(a_i-1)/(n_i(n_i-1))$, the unbiased
all-pairs estimators of $\theta_i$ and $\theta_i^2$. Feature maps: 1-D is
$[1,m,\dots,m^4]$; 2-D is the total-degree-$3$ tensor basis in $(m,s)$, $10$
features. We use ridge regression with $\lambda=10^{-3}$. Predictions are
clipped to $\widehat\eta_1\in[0,1]$ and projected to the admissible region
$\widehat\eta_1^{\,2}\le\widehat\eta_2\le\widehat\eta_1$.

\paragraph{Methods.} Eight ranking scores: raw $2m(1-m)$; raw $s$; raw $1-s$;
1-D first-order $2\widehat\eta_1^{1\mathrm{D}}(1-\widehat\eta_1^{1\mathrm{D}})$;
1-D second-order $2\widehat\eta_1^{1\mathrm{D}}-2\widehat\eta_2^{1\mathrm{D}}$;
2-D first-order $2\widehat\eta_1^{2\mathrm{D}}(1-\widehat\eta_1^{2\mathrm{D}})$;
2-D second-order $2\widehat\eta_1^{2\mathrm{D}}-2\widehat\eta_2^{2\mathrm{D}}$
(primary); oracle $g(\theta_i)=2\theta_i(1-\theta_i)$, shown for upper-bound
reference only.

\paragraph{Splits and reporting.} Per repeat, items are partitioned disjointly
into selection ($25\%$), calibration ($35\%$), evaluation (rest). Cohort
thresholds, $s_{\mathrm{gap}}$, noise scales, score-mode, and target class are
chosen on the selection split by a held-out audit-AUC lift criterion; the
final figure uses $1000$ fresh draws of the calibration/evaluation split at
the chosen configuration $(\text{extreme\_thr},\text{ale\_width},\sigma_m,
\sigma_s,s_{\mathrm{gap}})=(0.12,0.22,0.06,0.08,0.35)$. Yield at audit
fraction $f$ is $100\cdot\overline{g(\theta_i)}$ over the top-$\lceil f
n_{\mathrm{ev}}\rceil$ evaluation items by score; bands are $\pm 1.96\,
\mathrm{SEM}$ across repeats. Source code reproducing all numbers and figures
is included with the submission.

\begin{figure}[h]
  \centering
  \includegraphics[width=0.75\linewidth]{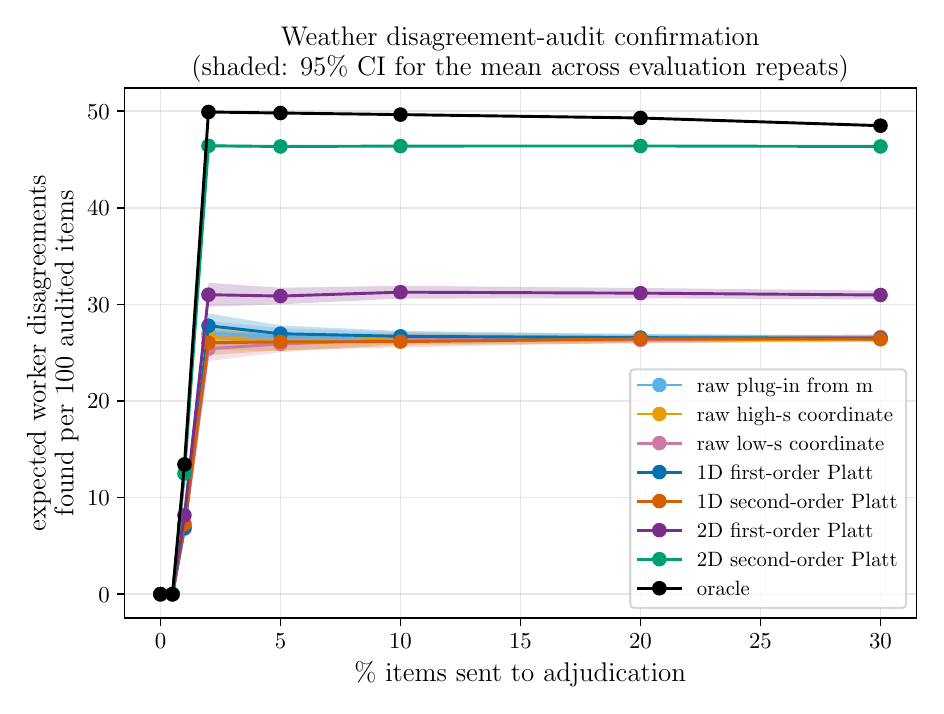}
  \caption{\textbf{Experiment~4: full audit-yield curves.} Yield (expected fresh-worker
  disagreements found per 100 audited items) vs.\ audit budget on the Weather Sentiment AMT
  cohort, over $1000$ evaluation repeats. 2-D second-order Platt is the only non-oracle method
  that exploits the joint $(m,s)$ structure required by the decision problem; all 1-D and raw
  baselines collapse onto a common curve because the 1-D marginals are identical between the
  aleatoric and hidden-subtype cohorts by construction. Bands are $\pm 1.96\,$SEM.}
  \label{fig:exp4}
\end{figure}

\paragraph{Per-method audit-yield summary.} At a $5\%$ audit budget the seven non-oracle methods
are well separated into three tiers: 2-D second-order Platt at $46.4$ disagreements per 100
audited items; 2-D first-order Platt at $30.9$; and the five 1-D and raw baselines (raw $m$, raw
$s$, raw $1-s$, 1-D first-order, 1-D second-order) clustered around ${\sim}26$. The oracle
$g(\theta_i)=2\theta_i(1-\theta_i)$ achieves $49.8$. The paired AUC lift of 2-D second-order
Platt over each of the six non-oracle baselines has a $95\%$ CI strictly above zero, with
per-baseline win fractions $\geq 0.97$ across the $1000$ evaluation repeats — i.e., the
2-D-second-order ordering audit-AUC dominates each baseline's ordering on essentially every
draw, not just on average. The collapse of all 1-D and raw curves onto a common envelope is
expected by construction: the XOR score design makes the cohort indicator
information-theoretically unrecoverable from any single coordinate, so no 1-D calibrator can do
better than chance at distinguishing aleatoric from hidden-subtype items.

\section{Choosing the perturbation bandwidth}
\label{appdx:choose-h}
In practice, we recommend choosing $h$ from a small, prespecified dyadic grid,
for example
\[
    \mathcal H=\{2^{-4},2^{-5},\ldots,2^{-10}\}.
\]
On a held-out selection split, evaluate the paired change in AUROC, or
preferably the relevant deployment utility, relative to the unperturbed
predictor. Because the deployed predictor is randomized, this comparison can
be averaged over several perturbation draws. Select the largest $h$ whose
upper confidence bound on performance degradation is below a prespecified
tolerance. When no application-specific tolerance is available, a simple
default is the one-standard-error rule: choose the largest $h$ whose
performance is within one standard error of the unperturbed predictor.

This rule reflects the statistical tradeoff. Smaller $h$ changes the
predictor less, but narrows the analyticity strip and increases the required
polynomial degree and finite-sample constants. Larger $h$ makes estimation
easier, but perturbs the predictor more. Performance and calibration estimates
at neighboring grid values provide a direct sensitivity check.

If $h$ is selected on data independent of the calibration sample,
Proposition~\ref{thm:main} applies conditionally on the selected value. For a
grid of $M$ prespecified values, the fixed-$h$ tail bounds in the proof can
also be made simultaneous by applying each at failure probability
$\delta/M$ and taking a union bound. This replaces
$\log(1/\delta)$ by $\log(M/\delta)$, adding only a term of order
$\sqrt{\log M/n}$ to the high-probability estimation error.

After selecting $h$, the theoretical degree
\[
    l_{\mathrm{th}}
    =\left\lceil\frac{2\log n}{\log\theta_h}\right\rceil
\]
can be used as an upper cap. The degree and ridge parameter can then be
selected by held-out validation or cross-validation, as in
Appendix~\ref{app:exp1}.

\section{Why Structure Beyond Continuity Is Necessary}
\label{appdx:assumptions-needed}
Over the unrestricted class of data-generating distributions, the minimax
error for estimating the unperturbed $\CE$ does not vanish. To see how the
first-order obstruction transfers, restrict to $f^*(X)\in\{0,1\}$ and set
\[
    \sigma^2(X)=m(X)(1-m(X)).
\]
In this construction, conditioning on $S=(m,\sigma^2)$ is equivalent to
conditioning on $m$, while $v=m$ and $\eta_2(S)=\eta_1(S)$. Therefore
\[
    \CE
    =2\E\!\left[\left|\E[Y\mid m]-m\right|\right]
    =2\CEOne.
\]
A uniformly vanishing estimator of unrestricted unperturbed $\CE$ would thus give a uniformly consistent estimator, and hence a consistent test, for unrestricted first-order calibration. This contradicts the first-order
testing impossibility under continuity alone \citep[Proposition~9]{lee2023t}.

\section{Proofs}
\subsection{Auxiliary lemma}
\begin{lemma}[Positive real part]
\label{lem:posreal}
For any $a \in \R$ and $|b| < \pi/2$:
\[
  \mathrm{Re}\bigl(\sech(a + ib)\bigr) = \frac{\cosh(a)\cos(b)}{|\cosh(a+ib)|^2} > 0.
\]
\end{lemma}

\begin{proof}
$\cosh(a+ib) = \cosh(a)\cos(b) + i\sinh(a)\sin(b)$.  So 
\[
\mathrm{Re}(1/\cosh(a+ib)) = \mathrm{Re}(\overline{\cosh(a+ib)})/|\cosh(a+ib)|^2 = \cosh(a)\cos(b)/|\cosh(a+ib)|^2.
\]
Since $\cosh(a) > 0$ for all real $a$ and $\cos(b) > 0$ for $|b| < \pi/2$, the result follows.
\end{proof}

\subsection{Proof of analyticity of calibration functions (lemma \ref{thm:analytic})}
\label{sec-lemma-analytic}

\begin{proof}
We first give the argument for $\eta_1$ in 1D, and then give the separate 2D argument.

The perturbed calibration function is $\eta_1(t) = N(t)/D(t)$ (see appendix D.1.1 of the work by \citet{ciosek2026measuring}).  Here
\begin{align}
  N(t) &= \int_0^1 \etao_1(s)\,\sech\!\left(\frac{t - s}{h}\right)\dt \tilde p(s), \label{eq:Nt}\\
  D(t) &= \int_0^1 \sech\!\left(\frac{t - s}{h}\right)\dt \tilde p(s), \label{eq:Dt}
\end{align}
Here $d\tilde p(s):=Z(s,h)^{-1}\,dP_{\mathrm{orig}}(s)$, where
$P_{\mathrm{orig}}$ is the original score law. This is a finite positive
measure because $Z(\cdot,h)$ is continuous and strictly positive on
$[0,1]$. No density assumption is needed.

\paragraph{Analyticity of $N$ and $D$.}
For each fixed $s \in [0,1]$, the function $z \mapsto \sech((z-s)/h)$ is analytic in $|\mathrm{Im}(z)| < h\pi/2$ (since $\sech(w) = 1/\cosh(w)$ has its nearest poles at $w = \pm i\pi/2$, i.e., at $z = s \pm ih\pi/2$).  Moreover, for $z = t + ib$ with $t \in [0,1]$ and $|b| \le h\pi/2 - \delta$ (any $\delta > 0$):
\[
  |\sech((z-s)/h)| = \frac{1}{|\cosh((t-s)/h + ib/h)|} \le \frac{1}{\cos(b/h) \cdot 1} \le \frac{1}{\cos(\pi/2 - \delta/h)},
\]
which is a uniform bound in $s$ and $t$. Since $0\le\etao_1\le1$ and
$\tilde p$ is a finite measure, the integrands are dominated on compact
subsets of the strip by an integrable constant. Integration therefore
preserves analyticity, so $N(z)$ and $D(z)$ are analytic in the strip.

\paragraph{Non-vanishing of $D$.}
For $z = t + ib$ with $|b| < h\pi/2$:
\[
  \mathrm{Re}\bigl(D(t+ib)\bigr) = \int_0^1 \mathrm{Re}\!\left(\sech\!\left(\frac{t+ib-s}{h}\right)\right)\dt \tilde p(s) > 0,
\]
by Lemma~\ref{lem:posreal}: the integrand is strictly positive and
$\tilde p$ has positive mass. Therefore $\mathrm{Re}\,D(z)>0$,
so $D(z)\ne0$ and $\eta_1=N/D$ is analytic in the strip.

\paragraph{2D extension.}
Fix $0<\rho<h\pi/4$. Let $u=(u_1,u_2)=(m,\sigma^2)$ and $z=(z_1,z_2)$.  Absorb the original score law and the two truncation normalizers into a finite positive measure $d\mu(u)$ on the feasible set $\cF\subseteq[0,1]\times[0,1/4]$.  For $j=1,2$ write
\[
  K_z(u):=\prod_{r=1}^2 \sech\!\left(\frac{z_r-u_r}{h}\right),\qquad
  D(z):=\int_{\cF} K_z(u)\,d\mu(u),\qquad
  N_j(z):=\int_{\cF}\eta_j^{\rm orig}(u)K_z(u)\,d\mu(u).
\]
Then $\eta_j(z)=N_j(z)/D(z)$ on the real rectangle.  On compact subsets of
\[
        \Omega_\rho:=\{z\in\C^2: |\operatorname{Im} z_1|<\rho,
        |\operatorname{Im} z_2|<\rho\},
\]
the kernels are uniformly bounded and holomorphic, so dominated convergence gives joint holomorphy of $N_j$ and $D$.

It remains to show that $D$ does not vanish and that $N_j/D$ is bounded.  If $w=a+ib$ with $|b|<\pi/2$, then
\[
   |\arg \sech(w)|=|\arg \cosh(w)|
   =\left|\arctan\{\tanh(a)\tan(b)\}\right|\le |b|.
\]
Hence, for $z$ with $|\operatorname{Im} z_r|\le\rho$,
\[
   |\arg K_z(u)|\le \frac{|\operatorname{Im}z_1|+|\operatorname{Im}z_2|}{h}
   \le \frac{2\rho}{h}<\frac{\pi}{2},
\]
and therefore
\[
        \operatorname{Re} K_z(u)\ge \cos(2\rho/h)|K_z(u)|.
\]
Consequently
\[
   \operatorname{Re}D(z)
   =\int_{\cF}\operatorname{Re}K_z(u)\,d\mu(u)
   \ge \cos(2\rho/h)\int_{\cF}|K_z(u)|\,d\mu(u)>0,
\]
so $D(z)\ne0$.  Since $0\le\eta_j^{\rm orig}\le1$,
\[
   |N_j(z)|\le\int_{\cF}|K_z(u)|\,d\mu(u)
        \le M_\rho\operatorname{Re}D(z)
        \le M_\rho |D(z)|,
        \qquad M_\rho:=\sec(2\rho/h).
\]
Thus, on the closed product sub-strip $|\operatorname{Im}z_r|\le\rho$,
\[
        |\eta_j(z)|\le M_\rho.
\]
\end{proof}

\subsection{Bernstein Ellipse Theorem in our notation (lemma \ref{prop:bernstein} in main text) }
\label{sec-proof-bernstein}

\begin{proof}
Define
\[
g(x):=f\!\left(\frac{x+1}{2}\right), \qquad x\in[-1,1].
\]
Since \(f\) is analytic in the strip
\[
\{z\in\C: |\mathrm{Im}(z)|<\rho\},
\]
the function \(g\) is analytic in the strip
\[
S:=\{x\in\C: |\mathrm{Im}(x)|<2\rho\},
\]
because
\[
\mathrm{Im}\!\left(\frac{x+1}{2}\right)=\frac{\mathrm{Im}(x)}{2},
\]
so \(x\in S\) implies
\[
\left|\mathrm{Im}\!\left(\frac{x+1}{2}\right)\right|<\rho.
\]
Moreover, \(|g(x)|\le M\) on \(S\).

Now let \(B(a)\) be the Bernstein ellipse from Theorem 3.17 of the lecture notes by \citet{fawzi-notes}:
\[
B(a)=\left\{x+iy\in\C:
\frac{x^2}{\cosh^2(a\pi)}+\frac{y^2}{\sinh^2(a\pi)}<1\right\}.
\]
If \(z=x+iy\in B(a)\), then
\[
\frac{y^2}{\sinh^2(a\pi)}<1,
\]
hence
\[
|\mathrm{Im}(z)|=|y|<\sinh(a\pi).
\]
Choose \(a>0\) so that
\[
\sinh(a\pi)=2\rho.
\]
Then every \(z\in B(a)\) satisfies
\[
|\mathrm{Im}(z)|<2\rho,
\]
so
\[
B(a)\subset S.
\]
Therefore \(g\) is analytic on \(B(a)\) and satisfies
\[
\sup_{z\in B(a)} |g(z)|\le M.
\]

By the Bernstein ellipse Theorem \citep[Theorem 3.17]{fawzi-notes}, for every \(l\ge 0\) there exists a polynomial \(p_l\) of degree \(l\) such that
\[
\sup_{x\in[-1,1]} |g(x)-p_l(x)|
\le 2M\,\frac{c^{l+1}}{1-c},
\qquad c=e^{-a\pi}.
\]
Now define
\[
q_l(t):=p_l(2t-1), \qquad t\in[0,1].
\]
Then \(q_l\) is a polynomial of degree \(l\), and with \(x=2t-1\),
\[
\sup_{t\in[0,1]} |f(t)-q_l(t)|
=
\sup_{x\in[-1,1]} |g(x)-p_l(x)|.
\]

Finally, set \(\theta:=e^{a\pi}\). Since \(\sinh(a\pi)=2\rho\), we have
\[
\theta
=e^{a\pi}
=\sinh(a\pi)+\sqrt{1+\sinh^2(a\pi)}
=2\rho+\sqrt{4\rho^2+1},
\]
and therefore \(c=\theta^{-1}\). Hence
\[
2M\,\frac{c^{l+1}}{1-c}
=
2M\,\frac{\theta^{-(l+1)}}{1-\theta^{-1}}
=
2M\,\frac{\theta^{-l}}{\theta-1}.
\]
This proves the claim.
\end{proof}

\subsection{Proof of corollary~\ref{cor:approx} in the main text}
\label{sec-proof-cor}
\begin{proof}
Indeed, map $\mathcal R$ to $[-1,1]^2$ by $x=2m-1$ and
$y=8\sigma^2-1$. The mapped function is bounded by $M_\rho$ on a
product of Bernstein ellipses with common parameter $\theta_\rho$; for
the second coordinate the available strip is larger, so using the
first-coordinate value $\theta_\rho$ is conservative. The tensor
Chebyshev coefficients satisfy
$|a_{kl}|\le4M_\rho\theta_\rho^{-(k+l)}$. Hence the tensor Chebyshev
truncation $p_{j,\ell}$ satisfies
\[
   \|p_{j,\ell}-\eta_j\|_\infty\le a_\ell,
   \qquad
   a_\ell:=
   \frac{8M_\rho\theta_\rho}{(\theta_\rho-1)^2}
   \theta_\rho^{-\ell},
\]
where the sup norm is over $\mathcal R$.

Since $\eta_j\in[0,1]$, this implies
$p_{j,\ell}\in[-a_\ell,1+a_\ell]$ on $\mathcal R$. Define
\[
   q_{j,\ell}:=\frac{p_{j,\ell}+a_\ell}{1+2a_\ell}.
\]
This is a polynomial of the same degree and satisfies
$0\le q_{j,\ell}\le1$ on $\mathcal R$. Moreover,
\[
\begin{split}
   \|q_{j,\ell}-\eta_j\|_\infty
   &\le
   \frac{\|p_{j,\ell}-\eta_j\|_\infty
   +a_\ell\|1-2\eta_j\|_\infty}{1+2a_\ell} \\
   &\le \frac{2a_\ell}{1+2a_\ell}
   \le 2a_\ell
   =B_\rho\theta_\rho^{-\ell}.
\end{split}
\]
\end{proof}

\subsection{Covering Number Lemma}

\begin{lemma}
Let
\[
   V_l:=\operatorname{span}
   \{u_1^i u_2^j:0\le i,j\le l\}
\]
be the degree-$l$ tensor polynomial space on
$\mathcal R=[0,1]\times[0,1/4]$, equipped with
\[
   \|p\|_\infty:=\sup_{u\in\mathcal R}|p(u)|.
\]
Let
\[
   B_l:=\{p\in V_l:\|p\|_\infty\le1\},
   \qquad
   \mathcal P_l
   :=\{\mathrm{clip}_{[0,1]}(p):p\in B_l\}.
\]
Then, for every $0<\varepsilon\le1$,
\[
   N(\varepsilon,\mathcal P_l,\|\cdot\|_\infty)
   \le
   \left(\frac{3}{\varepsilon}\right)^{(l+1)^2}.
\]
Equivalently,
\[
   \log N(\varepsilon,\mathcal P_l,\|\cdot\|_\infty)
   \le
   (l+1)^2\log\left(\frac{3}{\varepsilon}\right).
\]
\label{lem-covering}
\end{lemma}

\begin{proof}
The space $V_l$ has dimension $d=(l+1)^2$. The standard volumetric
bound for the unit ball of a $d$-dimensional normed space gives
\[
   N(\varepsilon,B_l,\|\cdot\|_\infty)
   \le
   \left(1+\frac{2}{\varepsilon}\right)^d
   \le
   \left(\frac{3}{\varepsilon}\right)^d,
\]
where the last inequality uses $0<\varepsilon\le1$.

The scalar clipping map is $1$-Lipschitz. Therefore, for any
$p,q\in B_l$,
\[
   \|\mathrm{clip}_{[0,1]}(p)
     -\mathrm{clip}_{[0,1]}(q)\|_\infty
   \le
   \|p-q\|_\infty.
\]
Consequently, clipping every element of an $\varepsilon$-net of $B_l$
produces an $\varepsilon$-net of $\mathcal P_l$ of the same size.
Thus
\[
   N(\varepsilon,\mathcal P_l,\|\cdot\|_\infty)
   \le
   N(\varepsilon,B_l,\|\cdot\|_\infty)
   \le
   \left(\frac{3}{\varepsilon}\right)^{(l+1)^2}.
\]
Taking logarithms proves the second statement.
\end{proof}

\subsection{More Auxiliary Lemmas}
\begin{lemma}
We have
\[
\mathbb{E}[h_{\hat f}]
\le
4\varepsilon_l^2+\frac{406\,d\log(3n)+408}{n}.
\]
\label{lem-aux-1}
\end{lemma}
\begin{proof}
    To keep constants explicit, we use the covering bound from lemma \ref{lem-covering} in the form
\[
N(\delta,\mathcal P_l,\|\cdot\|_\infty)\le \Bigl(\frac{3}{\delta}\Bigr)^d,
\qquad 0<\delta\le 1,
\]
with \(d=(l+1)^2\).

Set
\[
\delta_n:=\frac1n,
\]
and let \(\mathcal G_n\subset \mathcal P_l\) be a \(\delta_n\)-cover of \(\mathcal P_l\) in
\(\|\cdot\|_\infty\). Then
\[
|\mathcal G_n|
\le
N(\delta_n,\mathcal P_l,\|\cdot\|_\infty)
\le
(3n)^d.
\]

For every \(f\in\mathcal P_l\), choose \(g_f\in\mathcal G_n\) such that
\[
\|f-g_f\|_\infty\le \delta_n.
\]
Recall that
\[
h_f(S,Z)=(f(S)-Z)^2-(f^*(S)-Z)^2.
\]
If \(f,g\in\mathcal P_l\), then
\[
h_f-h_g=(f-g)(f+g-2Z),
\]
hence, since \(f,g\in[0,1]\) and \(Z\in\{0,1\}\),
\[
|h_f-h_g|
\le
2|f-g|
\le
2\|f-g\|_\infty.
\]
Therefore
\[
|Ph_f-Ph_g|\le 2\|f-g\|_\infty,
\qquad
|P_n h_f-P_n h_g|\le 2\|f-g\|_\infty,
\]
where \(P_n h:=\frac1n\sum_{i=1}^n h(S_i,Z_i)\).
In particular, for the ERM \(\hat f\), since \(P_n h_{\hat f}\le 0\),
\[
P_n h_{g_{\hat f}}
\le
P_n h_{\hat f}+2\delta_n
\le
\frac{2}{n},
\]
and
\[
Ph_{\hat f}\le Ph_{g_{\hat f}}+2\delta_n.
\]

Now fix \(g\in\mathcal G_n\), and write
\[
\mu_g:=Ph_g=\mathbb E[h_g].
\]
By Step 2(b), \(h_g\in[-4,4]\), so
\[
(\mu_g-h_g)\le \mu_g+4\le 8.
\]
Applying Bernstein's inequality \citep[Corollary 7.3]{rebeschini-notes} to the random variables
\[
X_i:=\mu_g-h_g(S_i,Z_i),
\]
which satisfy \(X_i-\mathbb EX_i\le 8\), gives for every \(t>0\),
\[
\Pr\!\left(\mu_g-P_n h_g\ge t\right)
\le
\exp\!\left(
-\frac{nt^2}{2(\operatorname{Var}(h_g)+\frac{8}{3}t)}
\right).
\]
By Step 2(c),
\[
\operatorname{Var}(h_g)\le \mathbb E[h_g^2]\le 32\mu_g+64\varepsilon_l^2.
\]
Hence
\[
\Pr\!\left(\mu_g-P_n h_g\ge t\right)
\le
\exp\!\left(
-\frac{nt^2}{2(32\mu_g+64\varepsilon_l^2+\frac{8}{3}t)}
\right).
\]

Suppose now that
\[
\mu_g\ge \frac{4}{n}
\qquad\text{and}\qquad
\mu_g\ge 4\varepsilon_l^2.
\]
Then \(t:=\mu_g-\frac{2}{n}\) satisfies \(t\ge \mu_g/2\), and therefore
\[
\Pr\!\left(P_n h_g\le \frac{2}{n}\right)
\le
\exp\!\left(
-\frac{n(\mu_g/2)^2}{2(32\mu_g+64\varepsilon_l^2+\frac{8}{3}\mu_g)}
\right).
\]
Using \(64\varepsilon_l^2\le 16\mu_g\), the denominator is at most
\[
2\left(32\mu_g+16\mu_g+\frac{8}{3}\mu_g\right)
=
\frac{304}{3}\mu_g,
\]
so
\[
\Pr\!\left(P_n h_g\le \frac{2}{n}\right)
\le
\exp\!\left(
-\frac{3n\mu_g}{1216}
\right)
\le
\exp\!\left(
-\frac{n\mu_g}{406}
\right).
\]

Now fix \(u\ge 0\), and define
\[
r(u):=
\max\!\left\{
\frac{4}{n},
\;
4\varepsilon_l^2,
\;
\frac{406\bigl(d\log(3n)+u\bigr)}{n}
\right\}.
\]
If \(\mu_g\ge r(u)\), then
\[
\Pr\!\left(P_n h_g\le \frac{2}{n}\right)
\le
\exp\!\left(
-\frac{n r(u)}{406}
\right)
\le
\exp\!\left(-d\log(3n)-u\right)
=
(3n)^{-d}e^{-u}.
\]
Taking a union bound over all \(g\in\mathcal G_n\), and using \(|\mathcal G_n|\le (3n)^d\), we obtain
\[
\Pr\!\left(
\exists g\in\mathcal G_n:\;
Ph_g\ge r(u)
\ \text{and}\ 
P_n h_g\le \frac{2}{n}
\right)
\le e^{-u}.
\]

On the complementary event, the particular net point \(g_{\hat f}\) associated with \(\hat f\) cannot satisfy
\(Ph_{g_{\hat f}}\ge r(u)\), because we already know \(P_n h_{g_{\hat f}}\le 2/n\). Thus
\[
Ph_{g_{\hat f}}<r(u),
\]
and hence
\[
Ph_{\hat f}\le Ph_{g_{\hat f}}+\frac{2}{n}< r(u)+\frac{2}{n}.
\]
We have proved that for every \(u\ge 0\),
\[
\Pr\!\left(
Ph_{\hat f}>
\frac{2}{n}
+
\max\!\left\{
\frac{4}{n},
\;
4\varepsilon_l^2,
\;
\frac{406\bigl(d\log(3n)+u\bigr)}{n}
\right\}
\right)
\le e^{-u}.
\]
Since
\[
\max\!\left\{
\frac{4}{n},
\;
4\varepsilon_l^2,
\;
\frac{406\bigl(d\log(3n)+u\bigr)}{n}
\right\}
\le
4\varepsilon_l^2+\frac{406\,d\log(3n)}{n}+\frac{406u}{n},
\]
we get
\[
\Pr\!\left(
Ph_{\hat f}>
4\varepsilon_l^2+\frac{406\,d\log(3n)+2}{n}+\frac{406u}{n}
\right)
\le e^{-u}.
\]
Integrating this tail bound,
\[
\mathbb E[Ph_{\hat f}]
\le
4\varepsilon_l^2+\frac{406\,d\log(3n)+2}{n}
+\frac{406}{n}\int_0^\infty e^{-u}\,du,
\]
hence
\[
\mathbb E[h_{\hat f}]
=
\mathbb E[Ph_{\hat f}]
\le
4\varepsilon_l^2+\frac{406\,d\log(3n)+408}{n}.
\]
This is the desired fast-rate bound.
\end{proof}

\subsection{Proof of Main Result (Proposition \ref{thm:main})}
\label{sec-main-proof}

\begin{proof}
Since both $\hat\eta_1$ and $\hat\eta_2$ are ERM in the same class $\Fl$ with $\{0,1\}$-valued responses, we give the argument once for a generic calibration function $\eta$ estimated by $\hat f = \arg\min_{f \in \Fl} (1/n)\sum (Z_i - f(S_i))^2$, where $Z_i \in \{0,1\}$.

\paragraph{Step 1: Approximation error.}
By Corollary~\ref{cor:approx}, there exists a degree-$l$ polynomial
$q_l$ satisfying $0\le q_l\le1$ on $\mathcal R$ and
\[
   \|q_l-\eta\|_\infty
   \le \varepsilon_l
   \le B_h\theta_h^{-l}
   \le B_hn^{-2}.
\]
Since $\|q_l\|_\infty\le1$ and
$\mathrm{clip}_{[0,1]}(q_l)=q_l$, we have $q_l\in\Fl$. Set
$f^*:=q_l$.

\paragraph{Step 2: Excess risk via the Bernstein condition.}
Define the excess loss $h_f(S, Z) = (f(S) - Z)^2 - (f^*(S) - Z)^2$.  We verify three properties:
\begin{enumerate}
\item[(a)] \emph{Excess risk equals excess $L^2$ error:}
$\E[h_f] = \norm{f - \eta}_{L^2(P)}^2 - \norm{f^* - \eta}_{L^2(P)}^2$.

This follows by expanding: $h_f = (f - f^*)(f + f^* - 2Z)$, so $\E[h_f \mid S] = (f - f^*)(f + f^* - 2\eta) = (f - \eta)^2 - (f^* - \eta)^2$.

\item[(b)] \emph{Boundedness:} $|h_f| \le |f - f^*| \cdot |f + f^* - 2Z| \le 1 \cdot 4 = 4$, since $f, f^* \in [0,1]$ and $Z \in \{0,1\}$.

\item[(c)] \emph{Bernstein condition:} $\E[h_f^2] \le 16\norm{f - f^*}_{L^2(P)}^2$, since $h_f^2 \le 16(f - f^*)^2$.  Furthermore,
\begin{align*}
  \norm{f - f^*}_{L^2}^2 &= \norm{(f - \eta) - (f^* - \eta)}_{L^2}^2 \\
  &\le 2\norm{f - \eta}_{L^2}^2 + 2\norm{f^* - \eta}_{L^2}^2 \\
  &= 2\E[h_f] + 2\norm{f^* - \eta}_{L^2}^2 + 2\norm{f^* - \eta}_{L^2}^2 \\
  &\le 2\E[h_f] + 4\varepsilon_l^2.
\end{align*}
So $\E[h_f^2] \le 32\E[h_f] + 64\varepsilon_l^2$.
\end{enumerate}

\paragraph{Step 3: Fast rate with explicit constants.}

Since, by Step 2(a),
\[
\mathbb{E}[h_{\hat f}]
=
\|\hat f-\eta\|_{L^2(P)}^2-\|f^*-\eta\|_{L^2(P)}^2,
\]
and \(\|f^*-\eta\|_{L^2(P)}\le \varepsilon_l\), it follows from lemma \ref{lem-aux-1} that
\[
\mathbb{E}\|\hat f-\eta\|_{L^2(P)}^2
\le
5\varepsilon_l^2+\frac{406\,d\log(3n)+408}{n}.
\]

\paragraph{Step 4: From \(L^2\) error to \(CE_2\).}
Apply the bound from Step 3 separately to the two calibration functions
\(\eta_1,\eta_2\). For \(j\in\{1,2\}\), let
\[
A_n:=5\varepsilon_l^2+\frac{406\,d\log(3n)+408}{n},
\]
so that
\[
\mathbb E\bigl[\|\hat\eta_j-\eta_j\|_{L^2(P)}^2\bigr]\le A_n.
\]
By Cauchy--Schwarz and then Jensen's inequality,
\[
\mathbb E\bigl[\|\hat\eta_j-\eta_j\|_{L^1(P)}\bigr]
\le
\mathbb E\bigl[\|\hat\eta_j-\eta_j\|_{L^2(P)}\bigr]
\le
\Bigl(\mathbb E\bigl[\|\hat\eta_j-\eta_j\|_{L^2(P)}^2\bigr]\Bigr)^{1/2}
\le \sqrt{A_n}.
\]

Now write
\[
\widetilde v:=\widetilde m^{\,2}+\widetilde\sigma_f^2.
\]
The perturbed second-order calibration error is
\[
CE^{\mathrm{pert}}_2
=
\mathbb E\bigl[|\eta_1(\widetilde S)-\widetilde m|\bigr]
+
\mathbb E\bigl[|\eta_2(\widetilde S)-\widetilde v|\bigr].
\]
Define the plug-in population analogue
\[
CE^{\mathrm{plug}}_2
:=
\mathbb E\bigl[|\hat\eta_1(\widetilde S)-\widetilde m|\bigr]
+
\mathbb E\bigl[|\hat\eta_2(\widetilde S)-\widetilde v|\bigr].
\]
Using the reverse triangle inequality,
\[
\bigl||a-q|-|b-q|\bigr|\le |a-b|,
\]
we obtain
\[
\bigl|CE^{\mathrm{plug}}_2-CE^{\mathrm{pert}}_2\bigr|
\le
\|\hat\eta_1-\eta_1\|_{L^1(P)}
+
\|\hat\eta_2-\eta_2\|_{L^1(P)}.
\]
Taking expectations and using the previous bound gives
\[
\mathbb E\bigl[\bigl|CE^{\mathrm{plug}}_2-CE^{\mathrm{pert}}_2\bigr|\bigr]
\le
2\sqrt{A_n}.
\]
That is,
\[
\mathbb E\bigl[\bigl|CE^{\mathrm{plug}}_2-CE^{\mathrm{pert}}_2\bigr|\bigr]
\le
2\left(
5\varepsilon_l^2+\frac{406\,d\log(3n)+408}{n}
\right)^{1/2}.
\]

Finally, with
\[
l=\Bigl\lceil \frac{2\log n}{\log\theta_h}\Bigr\rceil,
\qquad
d=(l+1)^2,
\]
and \(\varepsilon_l\le B_h \theta_h^{-l}\) from Corollary~\ref{cor:approx}, we have
\[
\varepsilon_l^2\le B_h^2\,\theta_h^{-2l}\le B_h^2 n^{-4},
\]
and therefore
\[
\mathbb E\bigl[\bigl|CE^{\mathrm{plug}}_2-CE^{\mathrm{pert}}_2\bigr|\bigr]
\le
2\left(
5B_h^2 n^{-4}
+
\frac{406\,(l+1)^2\log(3n)+408}{n}
\right)^{1/2}.
\]
In particular,
\[
\mathbb E\bigl[\bigl|CE^{\mathrm{plug}}_2-CE^{\mathrm{pert}}_2\bigr|\bigr]
\le
2\left(
5B_h^2 n^{-4}
+
\frac{406\left(\frac{2\log n}{\log\theta_h}+2\right)^2\log(3n)+408}{n}
\right)^{1/2}.
\]
This proves the
$O_h(\log^{3/2}n/\sqrt n)$ bound for
$\mathrm{CE}_2^{\mathrm{plug}}$.

\paragraph{Step 5: From population to empirical CE.}
The preceding steps bound
$\mathbb{E}\bigl[\bigl|\mathrm{CE}_2^{\mathrm{plug}}-\mathrm{CE}_2^{\mathrm{pert}}\bigr|\bigr]$,
where
$\mathrm{CE}_2^{\mathrm{plug}}=\mathbb{E}\bigl[|\hat\eta_1(\tilde S)-\tilde m|\bigr]
+\mathbb{E}\bigl[|\hat\eta_2(\tilde S)-\tilde v|\bigr]$
is evaluated under the \emph{population} measure.
The estimator~\eqref{eq:cehat} uses the \emph{empirical} measure over the same sample, so we must control the difference.

For any fixed $f\in\mathcal{P}_l$ the map
$s\mapsto |f(s)-\tilde m|$ takes values in $[0,1]$.
Define the function class
\[
\mathcal{G}_l \;=\; \bigl\{\, (s,\tilde m,\tilde v)
  \mapsto |f(s)-\tilde m| + |g(s)-\tilde v|
  \;:\; f,g\in\mathcal{P}_l \bigr\}.
\]
Every member of $\mathcal{G}_l$ is bounded in $[0,\frac94]$.
We now bound its covering number.

For any fixed constant $c$, the map $t\mapsto|t-c|$ is $1$-Lipschitz,
so for any $f_1,f_2\in\mathcal{P}_l$ and any $c$,
\[
\sup_{s}\bigl|\,|f_1(s)-c|-|f_2(s)-c|\,\bigr|
  \;\le\; \sup_{s}|f_1(s)-f_2(s)|
  \;=\; \|f_1-f_2\|_\infty.
\]
Therefore, if $\{f^{(1)},\dots,f^{(N_1)}\}$ is an $\epsilon$-cover of
$\mathcal{P}_l$ in $\|\cdot\|_\infty$, then for every $f\in\mathcal{P}_l$
there exists $f^{(k)}$ with
$\sup_s\bigl|\,|f(s)-c|-|f^{(k)}(s)-c|\,\bigr|\le\epsilon$.

Now consider a pair $(f,g)\in\mathcal{P}_l\times\mathcal{P}_l$ and
choose respective $\epsilon$-approximants $f^{(k)},g^{(j)}$.
By the triangle inequality applied to the sum defining $\mathcal{G}_l$,
\[
\bigl\|\,
  (|f(\cdot)-\tilde m|+|g(\cdot)-\tilde v|)
  -(|f^{(k)}(\cdot)-\tilde m|+|g^{(j)}(\cdot)-\tilde v|)
\,\bigr\|_\infty
\;\le\;
\|f-f^{(k)}\|_\infty + \|g-g^{(j)}\|_\infty
\;\le\; 2\epsilon.
\]
Hence an $\epsilon$-cover of $\mathcal{P}_l$ of size $N_1$ induces a
$2\epsilon$-cover of $\mathcal{G}_l$ of size $N_1^2$.
Substituting $\epsilon/2$ and applying lemma~4:
\[
\log N(\epsilon,\,\mathcal{G}_l,\,\|\cdot\|_\infty)
  \;\le\; 2\log N\!\bigl(\tfrac{\epsilon}{2},\,\mathcal{P}_l,\,\|\cdot\|_\infty\bigr)
  \;\le\; 2(l+1)^2\log\frac{6}{\epsilon}
  \;=\; 2d\log\frac{6}{\epsilon}.
\]
We now bound
\[
U \;:=\;
  \sup_{h\in\mathcal{G}_l}
  \left|
    \frac{1}{n}\sum_{i=1}^{n} h(Z_i)
    - \mathbb{E}[h(Z)]
  \right|,
\]
where $Z_i = (S_i, \tilde m_i, \tilde v_i)$ are i.i.d.

\medskip
\noindent\emph{Discretisation.}
Fix $\epsilon = 1/n$ and let $\mathcal{N}\subset\mathcal{G}_l$ be
an $(1/n)$-cover of $\mathcal{G}_l$ in $\|\cdot\|_\infty$.
By the covering-number bound derived above,
\[
|\mathcal{N}|
  \;\le\; N\!\bigl(\tfrac{1}{n},\,\mathcal{G}_l,\,\|\cdot\|_\infty\bigr)
  \;\le\; (6n)^{2(l+1)^2}.
\]

\noindent\emph{Concentration on the cover.}
For any fixed $h\in\mathcal{N}$, the random variables
$h(Z_1),\dots,h(Z_n)$ are i.i.d.\ in $[0,9/4]$.
By Hoeffding's inequality (with range $\frac94 -0=\frac94$), for every $t>0$:
\[
\Pr\!\left(
  \left|\frac{1}{n}\sum_{i=1}^{n}h(Z_i) - \mathbb{E}[h(Z)]\right| > t
\right)
\;\le\; 2\exp\!\left(-\frac{2nt^2}{(\frac94)^2}\right) = 2\exp\!\left(-\frac{32}{81}nt^2\right).
\]
A union bound over $\mathcal{N}$ gives
\begin{align}
    \label{eq:cover-conc}
\Pr\!\left(
  \max_{h\in\mathcal{N}}
  \left|\frac{1}{n}\sum_{i=1}^{n}h(Z_i) - \mathbb{E}[h(Z)]\right| > t
\right)
&\;\le\; 2\,|\mathcal{N}|\,\exp\!\left(-\frac{32}{81}nt^2\right) \\
&\;\le\; 2\,(6n)^{2(l+1)^2}\exp\!\left(-\frac{32}{81}nt^2\right).
\end{align}

\noindent\emph{Extension to all of $\mathcal{G}_l$.}
For any $h\in\mathcal{G}_l$, choose $h'\in\mathcal{N}$ with
$\|h-h'\|_\infty \le 1/n$.  Then
\[
\left|\frac{1}{n}\sum_i h(Z_i) - \mathbb{E}[h(Z)]\right|
\;\le\;
\left|\frac{1}{n}\sum_i h'(Z_i) - \mathbb{E}[h'(Z)]\right|
+ 2\|h-h'\|_\infty
\;\le\;
\left|\frac{1}{n}\sum_i h'(Z_i) - \mathbb{E}[h'(Z)]\right|
+ \frac{2}{n},
\]
where the first inequality is the triangle inequality applied to
both the empirical and population terms.
Taking the supremum over $h$, and combining
with~\eqref{eq:cover-conc}, we obtain a tail bound for
$U := \sup_{h\in\mathcal{G}_l}
\bigl|\frac{1}{n}\sum_{i=1}^{n} h(Z_i) - \mathbb{E}[h(Z)]\bigr|$:
\begin{equation}\label{eq:U-tail}
\Pr\!\left(U > t + \frac{2}{n}\right)
\;\le\; 2\,(6n)^{2(l+1)^2}\exp\!\left(-\frac{32}{81}nt^2\right),
\qquad t > 0.
\end{equation}

\noindent\emph{Expectation bound.}
Using $\mathbb{E}[U]=\int_0^\infty \Pr(U>s)\,ds$ and
the substitution $s = t + 2/n$,
\[
\mathbb{E}[U]
\;\le\;
\frac{2}{n}
+ \int_0^{\infty}
  \min\!\Bigl(1,\;
    2(6n)^{2(l+1)^2}\,e^{-\frac{32}{81}nt^2}
  \Bigr)\,dt.
\]
Define
\[
t^*
\;=\;
\sqrt{
  \frac{\frac{81}{16}(l+1)^2\log(6n) + \frac{81}{32}\log 2}{n}
},
\]
chosen so that
$2(6n)^{2(l+1)^2}\exp(-\frac{32}{81}n(t^*)^2) = 1$.
Splitting the integral at $t^*$:
\[
\int_0^{\infty}\min(\cdots)\,dt
\;\le\;
t^*
+ \int_{t^*}^{\infty} 2(6n)^{2(l+1)^2}\,e^{-\frac{32}{81}nt^2}\,dt.
\]
For the tail integral, note that for $t\ge t^* > 0$
we have $1 \le t/t^*$, so
\[
\int_{t^*}^{\infty} e^{-\frac{32}{81}nt^2}\,dt
\;\le\;
\frac{1}{t^*}\int_{t^*}^{\infty} t\,e^{-\frac{32}{81}nt^2}\,dt
\;=\;
\frac{1}{t^*}
\left[-\frac{81}{64n}e^{-\frac{32}{81}nt^2}\right]_{t^*}^{\infty}
\;=\;
\frac{81e^{-\frac{32}{81}n(t^*)^2}}{64nt^*}.
\]
Multiplying by the prefactor and using the definition of $t^*$:
\[
2(6n)^{2(l+1)^2}\cdot
\frac{81e^{-\frac{32}{81}n(t^*)^2}}{64nt^*}
\;=\;
\frac{81}{64nt^*}.
\]
Therefore
\begin{equation}\label{eq:EU-explicit}
\mathbb{E}[U]
\;\le\;
\frac{2}{n} + t^* + \frac{81}{64nt^*}
\;=\;
\frac{2}{n}
+ \sqrt{
  \frac{\frac{81}{16}(l+1)^2\log(6n) + \frac{81}{32}\log 2}{n}
}
+ \frac{81}{64\sqrt{n\bigl(\frac{81}{16}(l+1)^2\log(6n)+ \frac{81}{32}\log 2\bigr)}}.
\end{equation}

\noindent\emph{Application to $\widehat{\mathrm{CE}}_2$.}
Since
$(\hat\eta_1,\hat\eta_2)\in\mathcal{P}_l\times\mathcal{P}_l$
regardless of the data realisation (the ERM always returns an element
of the hypothesis class), the random function
$z\mapsto |\hat\eta_1(z)-\tilde m|+|\hat\eta_2(z)-\tilde v|$
lies in $\mathcal{G}_l$, and the supremum $U$ controls this particular
instance.  Concretely,
$|\widehat{\mathrm{CE}}_2 - \mathrm{CE}_2^{\mathrm{plug}}| \le U$,
so~\eqref{eq:EU-explicit} gives
\[
\mathbb{E}\bigl[\bigl|\widehat{\mathrm{CE}}_2
  - \mathrm{CE}_2^{\mathrm{plug}}\bigr|\bigr]
\;\le\;
\frac{2}{n}
+ \sqrt{
  \frac{\frac{81}{16}(l+1)^2\log(6n) + \frac{81}{32}\log 2}{n}
}
+ \frac{81}{64\sqrt{n\bigl(\frac{81}{16}(l+1)^2\log(6n)+ \frac{81}{32}\log 2\bigr)}}..
\]

\noindent\emph{Combining with Step~4.}
By the triangle inequality,
\[
\mathbb{E}\bigl[\bigl|\widehat{\mathrm{CE}}_2
  - \mathrm{CE}_2^{\mathrm{pert}}\bigr|\bigr]
\;\le\;
\mathbb{E}\bigl[\bigl|\widehat{\mathrm{CE}}_2
  - \mathrm{CE}_2^{\mathrm{plug}}\bigr|\bigr]
+
\mathbb{E}\bigl[\bigl|\mathrm{CE}_2^{\mathrm{plug}}
  - \mathrm{CE}_2^{\mathrm{pert}}\bigr|\bigr].
\]
The first term is bounded above.
The second term was bounded in Step~4:
\[
\mathbb{E}\bigl[\bigl|\mathrm{CE}_2^{\mathrm{plug}}
  - \mathrm{CE}_2^{\mathrm{pert}}\bigr|\bigr]
\;\le\;
2\!\left(
  5B_h^2 n^{-4}
  + \frac{406(l+1)^2\log(3n)+408}{n}
\right)^{\!1/2}.
\]
With $l = \lceil 2\log n/\log\theta_h\rceil$,
both terms are $O_h(\log^{3/2}n/\sqrt{n})$:
the dominant contribution in each is the
$\sqrt{(l+1)^2\log n\,/\,n}$ factor, and
$(l+1)^2=O_h(\log^2 n)$.
This proves the rate claimed in equation~\eqref{eq:rate}.
\end{proof}

\subsection{Proof of Corollary \ref{cor:approx_recal}}
\label{appdx:cor-recal}
\begin{proof}
Condition on the calibration data $\mathcal D$, including their
perturbations. Let $P$ denote the law of a fresh perturbed score $S$,
independent of $\mathcal D$. Until the final averaging step, all
expectations below are conditional on $\mathcal D$.

First control the effect of clipping the second moment. Fix a score $s$
and write $u=\eta_1(s)$ and $v=\eta_2(s)$. The true moments satisfy
$u^2\le v\le u$. Let $\pi_s$ be clipping to $[a(s)^2,a(s)]$.
Since clipping is $1$-Lipschitz,
\[
|b'(s)-v|
\le |b(s)-v|+|\pi_s(v)-v|.
\]
If $v$ lies inside this interval, the last term is zero. If $v<a(s)^2$,
it is at most
$|a(s)^2-u^2|\le2|a(s)-u|$; if $v>a(s)$, it is at most
$|u-a(s)|$. Therefore
\[
|b'(s)-\eta_2(s)|
\le |\hat\eta_2(s)-\eta_2(s)|
   +2|\hat\eta_1(s)-\eta_1(s)|.
\]

The first and second moments reported by $\hat T$ are $a(S)$ and
$b'(S)$, respectively. Both are functions of $\hat T$, since its
second reported moment equals its variance plus its squared mean.
The tower property and conditional Jensen inequality consequently give
\[
\begin{aligned}
\CE(\hat T)
&\le \E[|\eta_1(S)-a(S)|]
    +\E[|\eta_2(S)-b'(S)|]\\
&\le 3\|\hat\eta_1-\eta_1\|_{L^1(P)}
    +\|\hat\eta_2-\eta_2\|_{L^1(P)}.
\end{aligned}
\]
Finally, take expectation over $\mathcal D$. Step~4 of the proof of
Proposition~\ref{thm:main} in Appendix~\ref{sec-main-proof} bounds the
expected value of each $L^1(P)$ error by
$O_h(\log^{3/2}n/\sqrt n)$, proving the claim.
\end{proof}

\section{Moments vs Wasserstein Distance for Defining Calibration}
\label{app:reconciliation}

In this appendix we relate Definition~\ref{def-ho-cal} to the
$k^{\text{th}}$-order calibration framework of
\citet{ahdritz2025provable}, specializing to binary labels and $k=2$.
For this comparison, both criteria use the same fixed partition.
Our average moment error is quantitatively comparable to the corresponding
average Wasserstein error, but does not generally control the worst-case
error over all partition cells. Conditional expectations in
Definition~\ref{def-ho-cal} are understood as versions defined almost
surely under the score law; their values on null sets do not affect $\CE$.

\subsection{Preliminaries: symmetrized snapshots}

A \emph{$2$-snapshot} for an instance $x$ is a pair $(y_1,y_2)\in\{0,1\}^2$ with
$y_1,y_2\mid x \overset{\text{iid}}{\sim} f^*(x)$.
Following \citet{ahdritz2025provable}, we identify each snapshot with its \emph{empirical distribution}
(normalized histogram):
\begin{equation}\label{eq:unif-def}
  \mathrm{Unif}(y_1,y_2)(y) \;:=\; \frac{1}{2}\sum_{i=1}^{2}\mathbf{1}[y_i=y],
  \qquad y\in\{0,1\}.
\end{equation}
Each element of $\{0,1\}^2$ receives equal weight $1/k=1/2$ in the histogram; the name
$\mathrm{Unif}$ refers to this equal weighting, \emph{not} to the resulting distribution being
uniform over~$Y$.  Concretely,
\[
  (0,0)\mapsto (1,0),\qquad
  (0,1),(1,0)\mapsto \bigl(\tfrac12,\tfrac12\bigr),\qquad
  (1,1)\mapsto (0,1),
\]
where we write distributions as vectors $(p_0,p_1)$ over $\{0,1\}$.  Since the snapshot is
i.i.d., the order of the tuple is immaterial, so the symmetrised snapshot space is
\[
  Y^{(2)} = \bigl\{0,\,\tfrac12,\,1\bigr\} \;\subset\; \Delta Y,
\]
a set of three atoms.  Distributions over $Y^{(2)}$ are parametrised by a probability vector
$(q_0,q_{1/2},q_1)$.

\subsection{The two definitions}

\paragraph{Our definition (Definition~\ref{def-ho-cal}, specialized to a partition).}
Fix a partition $[\cdot]$ of~$\mathcal{X}$.  Define the calibration functions
\[
  \eta_1(s) = \mathbb{E}[Y \mid S(X)=s],\qquad
  \eta_2(s) = \mathbb{E}[f^*(X)^2 \mid S(X)=s],
\]
with $S(X)=(m(X),\sigma^2(X))$ and $v=m^2+\sigma^2$.  The second-order calibration error is
\begin{equation}\label{eq:CE2-def-app}
  \mathrm{CE}_2 = \mathbb{E}\bigl[|\eta_1(S)-m|\bigr]
               + \mathbb{E}\bigl[|\eta_2(S)-v|\bigr].
\end{equation}
When the conditioning is taken with respect to the partition (i.e.\
$\eta_1([x])=\mathbb{E}[f^*(X)\mid X\in[x]]$ and $\eta_2([x])=\mathbb{E}[f^*(X)^2\mid
X\in[x]]$), the pointwise moment errors on a partition element~$[x]$ are
\begin{equation}\label{eq:moment-errors}
  \delta_1([x]) = \bigl|\eta_1([x]) - m([x])\bigr|,\qquad
  \delta_2([x]) = \bigl|\eta_2([x]) - v([x])\bigr|.
\end{equation}

\paragraph{Ahdritz et al.'s definition~2.3.}
A $2$-snapshot predictor $g\colon\mathcal{X}\to\Delta Y^{(2)}$ is
\emph{$\varepsilon$-second-order calibrated} with respect to a partition~$[\cdot]$ if,
for all $x\in\mathcal{X}$,
\begin{equation}\label{eq:ahdritz-def}
  W_1\!\bigl(g(x),\;\mathrm{proj}_2 f^*([x])\bigr) \;\le\; \varepsilon,
\end{equation}
where $W_1$ denotes the Wasserstein-$1$ distance on $\Delta\Delta Y$ with the
$\ell_1$~distance on $\Delta Y$ as the ground metric, and $\mathrm{proj}_2 f^*([x])$ is
the true distribution of symmetrised $2$-snapshots drawn from~$[x]$.

\begin{remark}[Worst-case vs.\ average]
\citet{ahdritz2025provable} require \eqref{eq:ahdritz-def} for \emph{all} equivalence
classes~$[x]$ (worst-case).  They note in footnote~4 that a natural relaxation requires
\eqref{eq:ahdritz-def} only with high probability over equivalence classes drawn according to
the marginal $D_X$.  We call this the \emph{relaxed} variant: for a given $\delta>0$,
\begin{equation}\label{eq:ahdritz-relaxed}
  \Pr_{[x]\sim D_X}\!\Bigl[
    W_1\!\bigl(g(x),\;\mathrm{proj}_2 f^*([x])\bigr) > \varepsilon
  \Bigr] \;\le\; \delta.
\end{equation}
Our quantity $\mathrm{CE}_2$ is an \emph{average} (expectation under~$D_X$) and therefore
corresponds most naturally to this relaxed variant.
\end{remark}

\subsection{Relating the two definitions}

We now establish the precise relationship.  The key is a pointwise sandwich between the
Wasserstein error and the moment errors.

\begin{lemma}[Pointwise sandwich]\label{lem:sandwich}
Fix a partition element $[x]$.  Let $\delta_1([x]),\delta_2([x])$ be as in
\eqref{eq:moment-errors} and $W_1([x])$ as in \eqref{eq:ahdritz-def}.  Then
\begin{equation}\label{eq:sandwich}
  \delta_1([x]) + \delta_2([x])
  \;\le\; \frac{3}{2}\,W_1([x])
  \;\le\; 3\bigl(\delta_1([x])+\delta_2([x])\bigr).
\end{equation}
\end{lemma}

\begin{proof}
A distribution over $Y^{(2)}=\{0,\frac12,1\}$ is a vector $(q_0,q_{1/2},q_1)$.  The
predictor's output corresponds to $(q_0,q_{1/2},q_1)$ and the true snapshot distribution to
$(p_0,p_{1/2},p_1)$, with
\[
  p_1 = \mu_2, \quad p_{1/2} = 2(\mu_1-\mu_2),\quad p_0 = 1-2\mu_1+\mu_2,
\]
where $\mu_j = \mathbb{E}_{x\sim[x]}[f^*(x)^j]$. For the predictor's output, let
$m([x])$ and $v([x]) = m([x])^2 + \sigma^2([x])$ denote the shared values of $m$ and $v$
on the partition element~$[x]$.  The predicted snapshot distribution is
\[
  q_1 = v([x]),\quad q_{1/2} = 2\bigl(m([x])-v([x])\bigr),\quad
  q_0 = 1-2\,m([x])+v([x]).
\]

Write $\Delta_j = p_j - q_j$ for $j\in\{0,\frac12,1\}$, so that $\Delta_0+\Delta_{1/2}+\Delta_1=0$. We now compute $W_1([x])$.  Each atom of $Y^{(2)}$ is a distribution over $\{0,1\}$:
$0\leftrightarrow(1,0)$, $\frac12\leftrightarrow(\frac12,\frac12)$,
$1\leftrightarrow(0,1)$.  The ground metric is the $\ell_1$ distance inherited from $\Delta Y$:
\[
  d(0,\tfrac12)=1,\qquad d(\tfrac12,1)=1,\qquad d(0,1)=2.
\]
Since $d(0,1)=d(0,\frac12)+d(\frac12,1)$, it is never cheaper to transport mass directly
between the outer atoms than to route it through~$\frac12$.  The optimal plan is therefore
determined by two edge flows: let $\phi_1$ be the net signed flow from atom~$0$ to
atom~$\frac12$, and $\phi_2$ from atom~$\frac12$ to atom~$1$.  Mass balance at each atom
gives $\phi_1=\Delta_0$ and $\phi_2=\Delta_0+\Delta_{1/2}=-\Delta_1$, so
\[
  W_1([x]) \;=\; |\phi_1|\cdot d(0,\tfrac12)\;+\;|\phi_2|\cdot d(\tfrac12,1)
            \;=\; |\Delta_0|+|\Delta_1|.
\]
Meanwhile, the moment errors are
\[
  \delta_1([x]) = |\mu_1 - m| = \tfrac12|\Delta_{1/2}+2\Delta_1|
                = \tfrac12|\Delta_0-\Delta_1|, \qquad
  \delta_2([x]) = |\mu_2 - v| = |\Delta_1|.
\]
\emph{Upper bound.}  By the triangle inequality,
$|\Delta_0| \le |\Delta_0-\Delta_1|+|\Delta_1| = 2\delta_1+\delta_2$, so
\[
  W_1 = |\Delta_0|+|\Delta_1| \le 2\delta_1+2\delta_2.
\]
This is the same as
\[
\frac32 W_1 \leq 3\delta_1+3\delta_2
\]

\emph{Lower bound.}  We have $\delta_2 = |\Delta_1| \le W_1$ and
$\delta_1 = \frac12|\Delta_0-\Delta_1| \le \frac12(|\Delta_0|+|\Delta_1|)=\frac12 W_1$, so
$\delta_1+\delta_2\le\frac32 W_1$.
\end{proof}

We can now state the three main properties.

\begin{property}[Equivalence at zero error]\label{prop:exact}
$\mathrm{CE}_2=0$ if and only if, for $D_X$-almost every equivalence class~$[x]$,
$W_1([x])=0$.
\end{property}

\begin{proof}
By \eqref{eq:sandwich}, $W_1([x])=0$ iff $\delta_1([x])+\delta_2([x])=0$.  Since
$\mathrm{CE}_2 = \mathbb{E}_{[x]}[\delta_1([x])+\delta_2([x])]$, the expectation vanishes iff
the integrand vanishes $D_X$-a.e.
\end{proof}

\begin{property}[Ahdritz approximate calibration implies bounded $\mathrm{CE}_2$]
\label{prop:ahdritz-to-us}
If $g$ is $\varepsilon$-second-order calibrated in the sense of \eqref{eq:ahdritz-def}
(worst-case), then
\[
  \mathrm{CE}_2 \;\le\; \tfrac{3}{2}\,\varepsilon.
\]
\end{property}

\begin{proof}
For every $[x]$, the lower bound in \eqref{eq:sandwich} gives
$\delta_1([x])+\delta_2([x])\le\frac32\varepsilon$.  Averaging over~$D_X$ yields the claim.
\end{proof}

\begin{property}[Bounded $\mathrm{CE}_2$ implies relaxed Ahdritz calibration]
\label{prop:us-to-ahdritz}
If $\mathrm{CE}_2\le\varepsilon$, then for every $\delta>0$,
\[
  \Pr_{[x]\sim D_X}\!\bigl[W_1([x])>2\varepsilon/\delta\bigr]\;\le\;\delta.
\]
In particular, $g$ satisfies the relaxed variant \eqref{eq:ahdritz-relaxed} with parameters
$(\varepsilon',\delta)$ for $\varepsilon'=2\varepsilon/\delta$.

However, $\mathrm{CE}_2\le\varepsilon$ does \emph{not} in general imply a uniform (worst-case)
bound on $W_1([x])$ over all partition elements, since $\mathrm{CE}_2$ is an average quantity.
\end{property}

\begin{proof}
The upper bound in \eqref{eq:sandwich} gives
$W_1([x])\le 2\left( \delta_1([x])+\delta_2([x]) \right)$, so
\[
  \mathbb{E}_{[x]}[W_1([x])] \;\le\; 2\,\mathrm{CE}_2 \;\le\; 2\varepsilon.
\]
Markov's inequality yields the stated tail bound.
\end{proof}

\begin{remark}
If one adopts the relaxed (high-probability) variant of \citeauthor{ahdritz2025provable}'s definition,
then Properties~\ref{prop:ahdritz-to-us} and~\ref{prop:us-to-ahdritz} together show that the
two notions are quantitatively comparable: writing $\overline{W}_1 =
\mathbb{E}_{[x]}[W_1([x])]$ for the average Wasserstein error,
\[
  \tfrac{1}{2}\,\overline{W}_1
  \;\le\; \mathrm{CE}_2
  \;\le\; \tfrac{3}{2}\,\overline{W}_1.
\]
\end{remark}

\section{Matching Lower Bound on Rate}
\label{app:lowerboundproof}
In this section, we prove Proposition \ref{prop:lb}.
\begin{proof} \ 
\subsection{The two-point family}

Take $\mathcal X$ a singleton (suppress $X$) and use the boundary constant
predictor
\[
  S(x) \;\equiv\; (m_0,\, \sigma_0^2) \;:=\; (1,\,0).
\]
Since $m_0 = 1$ forces $m_0(1 - m_0) = 0$, the only valid value is
$\sigma_0^2 = 0$; the score is therefore the deterministic point $(1, 0)$,
which lies in the feasible region. The sech perturbation produces
$\St = (\mt, \sigt)$, where $\mt$ has density $k_h(\cdot \mid 1)$ on $[0,1]$
and $\sigt$ has density $k_h(\cdot \mid 0)$ on $[0, 1/4]$, independently. The
law of $\St$ is the \emph{same} under both $P_0$ and $P_1$.

For $b \in \{0, 1\}$, let $P_b$ be the law in which $f^* \equiv p_b$, with
\[
  p_0 \;:=\; \tfrac{1}{16}, \qquad p_1 \;:=\; \tfrac{1}{16} + \varepsilon,
\]
for an $\varepsilon \in (0, 1/16]$ to be chosen later. So
$Y_i^{(j)} \stackrel{\mathrm{iid}}{\sim} \Bern(p_b)$, independent of $\St$.

\subsection{The calibration error has a clean closed form}

Since $S$ is constant, $\St$ is independent of the labels and the calibration
functions are themselves constants:
\[
  \eta_1(\St) = p_b, \qquad \eta_2(\St) = p_b^2.
\]
(Both are trivially analytic, satisfying the structural assumptions of the
main paper.) Substituting:
\begin{equation}\label{eq:CE-form}
  \CEpert(P_b) \;=\; \underbrace{\E\bigl[|p_b - \mt|\bigr]}_{=:\,\phi(p_b)}
                  + \underbrace{\E\bigl[|p_b^2 - W|\bigr]}_{=:\,\psi(p_b)},
  \qquad W := \mt^{\,2} + \sigt.
\end{equation}

\subsection{Key lemma: stochastic dominance, uniform in \texorpdfstring{$h$}{h}}

\begin{lemma}\label{lem:dom}
For $m_0 = 1$ and every $h > 0$,
\[
  F_{\mt}(p) \;\leq\; p \qquad \text{for all } p \in [0,1].
\]
Equivalently, $\mt$ stochastically dominates $U[0,1]$.
\end{lemma}

\begin{proof}
Let $G(x) := \int_0^x \sech(v)\,dv$.
We have $G(0) = 0$, $G'(x) = \sech(x) > 0$, and
$G''(x) = -\sech(x)\tanh(x) \leq 0$ for $x \geq 0$, so $G$ is concave on
$[0, \infty)$.

Substituting $v = (1 - t)/h$ in the integral defining $F_{\mt}$:
\[
  F_{\mt}(p)
  \;=\; \frac{\int_0^p \sech\!\bigl((t - 1)/h\bigr)\,dt}
             {\int_0^1 \sech\!\bigl((u - 1)/h\bigr)\,du}
  \;=\; 1 - \frac{G\!\bigl((1 - p)/h\bigr)}{G(1/h)}.
\]
By concavity of $G$ on $[0, \infty)$ and $G(0) = 0$, for any $\theta \in [0,1]$
and $x \geq 0$,
\[
  G(\theta x) \;=\; G\bigl(\theta x + (1-\theta)\cdot 0\bigr)
  \;\geq\; \theta\, G(x) + (1-\theta)\, G(0)
  \;=\; \theta\, G(x).
\]
Setting $\theta = 1 - p$ and $x = 1/h$:
$G\bigl((1-p)/h\bigr) \geq (1-p)\, G(1/h)$, so
$F_{\mt}(p) \leq 1 - (1 - p) = p$.
\end{proof}

\subsection{The gap}

We bound $\phi$ and $\psi$ separately by integrating their derivatives.

\paragraph{First moment.} The function $\phi(p) = \E|p - \mt|$ is convex and
differentiable everywhere on $(0,1)$ (since $\mt$ has a continuous density),
with
\[
  \phi'(p) \;=\; 2 F_{\mt}(p) - 1 \;\leq\; 2p - 1
\]
by Lemma~\ref{lem:dom}. Integrating from $p_0$ to $p_1 = p_0 + \varepsilon$,
\begin{align*}
  \phi(p_1) - \phi(p_0)
  &\;\leq\; \int_{p_0}^{p_1} (2p - 1)\,dp
   \;=\; (p_1^2 - p_0^2) - (p_1 - p_0) \\
  &\;=\; \varepsilon\,(p_0 + p_1 - 1)
   \;\leq\; \varepsilon\,\bigl(\tfrac{1}{16} + \tfrac{1}{8} - 1\bigr)
   \;=\; -\tfrac{13}{16}\,\varepsilon,
\end{align*}
using $p_1 \leq 1/8$.

\paragraph{Second moment.} Differentiation under the expectation (justified by
the Lipschitz continuity in $p$ of $|p^2 - W|$ and the continuous density of
$W$) gives
\[
  \psi'(p) \;=\; 2p\,\bigl(2 F_W(p^2) - 1\bigr),
  \qquad |\psi'(p)| \leq 2p.
\]
Hence
\[
  \bigl|\psi(p_1) - \psi(p_0)\bigr|
  \;\leq\; \int_{p_0}^{p_1} 2p\,dp
  \;=\; p_1^2 - p_0^2
  \;\leq\; 2 p_1\,\varepsilon
  \;\leq\; \tfrac{\varepsilon}{4}.
\]

\paragraph{Combining.} From \eqref{eq:CE-form},
\begin{equation}\label{eq:gap}
  \CEpert(P_0) - \CEpert(P_1)
  \;=\; \bigl(\phi(p_0) - \phi(p_1)\bigr) + \bigl(\psi(p_0) - \psi(p_1)\bigr)
  \;\geq\; \tfrac{13}{16}\varepsilon - \tfrac{1}{4}\varepsilon
  \;=\; \tfrac{9}{16}\varepsilon
  \;\geq\; \tfrac{\varepsilon}{2}.
\end{equation}

\subsection{KL bound}

The perturbed scores carry no information about $b$ (their law is identical
under $P_0, P_1$ and they are independent of the labels), so each observation
under $P_b$ is informationally equivalent to two i.i.d.\ $\Bern(p_b)$ draws.
With both $p_0 = 1/16$ and $p_1 \leq 1/8$, $p_1(1 - p_1) \geq (1/16)(15/16)
\geq 1/18$, so
\[
  \KL\bigl(\Bern(p_0)\,\|\,\Bern(p_1)\bigr)
  \;\leq\; \frac{(p_1 - p_0)^2}{p_1(1 - p_1)}
  \;\leq\; 18\,\varepsilon^2.
\]
Tensorizing,
\begin{equation}\label{eq:KL}
  \KL\bigl(P_0^{\otimes n}\,\|\,P_1^{\otimes n}\bigr) \;\leq\; 36\,n\,\varepsilon^2.
\end{equation}

\subsection{Le Cam}

Le Cam's two-point lemma \citep[Chapter 31]{Polyanskiy_Wu_2025}: for any real-valued estimator $\widehat\theta$,
\[
  \max_{b\in\{0,1\}}
    \E_{P_b^{\otimes n}}\!\bigl[|\widehat\theta - \theta(P_b)|\bigr]
  \;\geq\; \tfrac12\,|\theta(P_1) - \theta(P_0)|\cdot
           \bigl(1 - \TV(P_0^{\otimes n}, P_1^{\otimes n})\bigr).
\]
By Pinsker and \eqref{eq:KL},
$\TV \leq \sqrt{\KL/2} \leq \sqrt{18\,n\,\varepsilon^2}
       = 3\,\varepsilon\,\sqrt{2n}$.

\paragraph{Choosing $\varepsilon$.} Set
\[
  \varepsilon \;:=\; \frac{1}{16\sqrt{n}}.
\]
Then $\TV \leq 1/2$, hence $1 - \TV \geq 1/2$. For $n \geq 1$ we have
$\varepsilon \leq 1/16$, so the gap bound \eqref{eq:gap} applies. Combining
\eqref{eq:gap} with Le Cam:
\[
  \max_{b\in\{0,1\}}
    \E_{P_b^{\otimes n}}\!\bigl[|\CEhat - \CEpert(P_b)|\bigr]
  \;\geq\; \tfrac12\cdot\tfrac{\varepsilon}{2}\cdot\tfrac12
  \;=\; \frac{\varepsilon}{8}
  \;=\; \frac{1}{128\sqrt{n}}.
\]
Setting $c := 1/128$, an absolute constant, independent of $h$, proves Proposition~\ref{prop:lb}.
\end{proof}

\section{Compute Resources}
The expected compute requirements are modest: all experiments can be run on a standard laptop. 

\section{Open Questions}
\label{app:openquestions}
We think the following questions would be useful avenues for further work.
\begin{enumerate}
    \item Extension to multiclass classification (we handle the binary case).
    \item Determine whether the $\log^{3/2}n$ factor in the upper bound is
    necessary. It arises from the structure of our proof, but is not present in the lower bound.
    \item Study the statistical and practical tradeoffs among perturbation,
    partitioning, and explicit structural assumptions on the calibration
    functions.
\end{enumerate}

\newpage
\section*{NeurIPS Paper Checklist}

\begin{enumerate}

\item {\bf Claims}
    \item[] Question: Do the main claims made in the abstract and introduction accurately reflect the paper's contributions and scope?
    \item[] Answer: \answerYes{}{} 
    \item[] Justification: The abstract and introduction accurately summarize the paper's main contributions: a bucket-free definition of second-order calibration, a near-$\tilde O(1/\sqrt{n})$ estimator based on analyticity of the $\sech$ perturbation kernel, and a second-order analogue of Platt scaling. These claims are matched by the formal statements and empirical results in the main text.

\item {\bf Limitations}
    \item[] Question: Does the paper discuss the limitations of the work performed by the authors?
    \item[] Answer: \answerYes 
    \item[] Justification: The paper discusses the main assumptions and practical limitations of the approach, including the reliance on 2-snapshots, the use of $\sech$ perturbation, the dependence on the bandwidth choice, and the current focus on binary classification. 

\item {\bf Theory assumptions and proofs}
    \item[] Question: For each theoretical result, does the paper provide the full set of assumptions and a complete (and correct) proof?
    \item[] Answer: \answerYes{} 
    \item[] Justification: All theoretical results are stated together with their assumptions, and complete proofs are provided in the appendix. The main theoretical guarantees appear in Lemma~\ref{thm:analytic}, Lemma~\ref{prop:bernstein}, Proposition~\ref{thm:main}, and Proposition~\ref{thm:recal}.


    \item {\bf Experimental result reproducibility}
    \item[] Question: Does the paper fully disclose all the information needed to reproduce the main experimental results of the paper to the extent that it affects the main claims and/or conclusions of the paper (regardless of whether the code and data are provided or not)?
    \item[] Answer: \answerYes{} 
    \item[] Justification: The paper provides full details for the synthetic and real-data experiments, including model architecture, datasets, snapshot construction, baselines, evaluation metrics, hyperparameters, and perturbation settings. The paper describes the setup sufficiently for reproducing the reported results.


\item {\bf Open access to data and code}
    \item[] Question: Does the paper provide open access to the data and code, with sufficient instructions to faithfully reproduce the main experimental results, as described in supplemental material?
    \item[] Answer: \answerYes 
    \item[] Justification: The code is included in the supplementary material. The paper also describes the datasets and experimental setup in enough detail to support faithful reproduction.

\item {\bf Experimental setting/details}
    \item[] Question: Does the paper specify all the training and test details (e.g., data splits, hyperparameters, how they were chosen, type of optimizer) necessary to understand the results?
    \item[] Answer: \answerYes{} 
    \item[] Justification: The paper specifies the training and evaluation protocol, including datasets, train/test splits, ensemble sizes, perturbation bandwidths, polynomial degrees, optimizers, and other hyperparameters used in the experiments. 

\item {\bf Experiment statistical significance}
    \item[] Question: Does the paper report error bars suitably and correctly defined or other appropriate information about the statistical significance of the experiments?
    \item[] Answer: \answerYes 
    \item[] Justification: The experimental results include confidence intervals and repeated-trial statistics for the key quantitative comparisons.

\item {\bf Experiments compute resources}
    \item[] Question: For each experiment, does the paper provide sufficient information on the computer resources (type of compute workers, memory, time of execution) needed to reproduce the experiments?
    \item[] Answer: \answerYes 
    \item[] Justification: The appendix describes the compute requirements for the experiments. 
    
\item {\bf Code of ethics}
    \item[] Question: Does the research conducted in the paper conform, in every respect, with the NeurIPS Code of Ethics \url{https://neurips.cc/public/EthicsGuidelines}?
    \item[] Answer: \answerYes 
    \item[] Justification: To the best of our knowledge, the research described in the paper conforms to the NeurIPS Code of Ethics. The work is primarily methodological and empirical, uses standard public benchmarks and existing annotations, and does not involve new human-subject interventions.

\item {\bf Broader impacts}
    \item[] Question: Does the paper discuss both potential positive societal impacts and negative societal impacts of the work performed?
    \item[] Answer: \answerNA 
    \item[] Justification: This paper is primarily a theoretical and methodological contribution on second-order calibration. It does not introduce a specific deployed system or application domain, so broader societal impacts are indirect rather than central to the paper's main contribution.
    
\item {\bf Safeguards}
    \item[] Question: Does the paper describe safeguards that have been put in place for responsible release of data or models that have a high risk for misuse (e.g., pre-trained language models, image generators, or scraped datasets)?
    \item[] Answer: \answerNA{} 
    \item[] Justification: The paper does not introduce or release a high-risk model or scraped dataset. Its main contribution is a calibration method and associated theoretical analysis.


\item {\bf Licenses for existing assets}
    \item[] Question: Are the creators or original owners of assets (e.g., code, data, models), used in the paper, properly credited and are the license and terms of use explicitly mentioned and properly respected?
    \item[] Answer: \answerYes 
    \item[] Justification: The paper cites the datasets and models used in the experiments.

\item {\bf New assets}
    \item[] Question: Are new assets introduced in the paper well documented and is the documentation provided alongside the assets?
    \item[] Answer: \answerNo 
    \item[] Justification: The paper introduces a new model/methodological contribution, but it does not release a corresponding software or model asset alongside the submission. The method and experimental setup are described in the paper, but no separately distributed asset is provided.

\item {\bf Crowdsourcing and research with human subjects}
    \item[] Question: For crowdsourcing experiments and research with human subjects, does the paper include the full text of instructions given to participants and screenshots, if applicable, as well as details about compensation (if any)? 
    \item[] Answer: \answerNA{} 
    \item[] Justification: The paper does not report new crowdsourcing experiments or new human-subject data collection by the authors. It uses an existing public dataset with human annotations.

\item {\bf Institutional review board (IRB) approvals or equivalent for research with human subjects}
    \item[] Question: Does the paper describe potential risks incurred by study participants, whether such risks were disclosed to the subjects, and whether Institutional Review Board (IRB) approvals (or an equivalent approval/review based on the requirements of your country or institution) were obtained?
    \item[] Answer: \answerNA{} 
    \item[] Justification: The paper does not involve new human-subjects research conducted by the authors. It relies on existing public annotated data and does not introduce new participant interaction or data collection.

\item {\bf Declaration of LLM usage}
    \item[] Question: Does the paper describe the usage of LLMs if it is an important, original, or non-standard component of the core methods in this research? Note that if the LLM is used only for writing, editing, or formatting purposes and does \emph{not} impact the core methodology, scientific rigor, or originality of the research, declaration is not required.
    \item[] Answer: \answerNA 
    \item[] Justification: LLMs are not an important, original, or non-standard component of the core methodology in this research.


\end{enumerate}

\end{document}